\documentclass{article}
\usepackage{nips14submit_e,times}
\usepackage[
    pdfkeywords={Clustering, Combinatorial Optimization, Convex Optimization, Graphical Models, Mathematical Programming, Mixed-Integer Nonlinear Programming, Ranking},
    pdftitle={Estimating Maximally Probable Constrained Relations by Mathematical Programming},
    pdfauthor={Lizhen Qu, Bjoern Andres}
]{hyperref}
\usepackage{url}
\usepackage{amsmath}
\usepackage{amssymb}
\usepackage{multicol}
\usepackage{multirow}
\usepackage{anyfontsize}
\usepackage{microtype}
\usepackage[font+=normal]{subcaption}
\usepackage{tikz}
    \usetikzlibrary{plotmarks}
    \usetikzlibrary{arrows}
    \tikzstyle{node}=[circle, draw=black, fill=white, inner sep=0pt, minimum width=1.7ex]
    \tikzstyle{cnode}=[circle, draw=black, fill=black!20, inner sep=0pt, minimum width=1.7ex]
    \tikzstyle{edge}=[draw=black,>=latex,->]
\usepackage{pgfplots}
\usepackage[thmmarks,amsmath]{ntheorem}
    \newtheorem{definition}{Definition}
    \newtheorem{lemma}{Lemma}
    
    \newtheorem*{proof}{Proof}
\newcommand{\eqnum}{\refstepcounter{equation}\textup{(\theequation)}}
\newcommand{\vi}{\mathrm{VI}}
\newcommand{\ri}{\mathrm{RI}}
\newcommand{\kl}{\mathrm{KL}}

\newcommand{\argmax}[1]{\underset{#1}{\mathrm{argmax}}}
\newcommand{\argmin}[1]{\underset{#1}{\mathrm{argmin}}}

\title{Estimating Maximally Probable Constrained Relations by Mathematical Programming}

\author{
Lizhen Qu\\
Max Planck Institute for Informatics\\
Saarbr\"ucken, Germany \\
\texttt{lqu@mpi-inf.mpg.de} \\
\And
Bjoern Andres\\
Max Planck Institute for Informatics\\
Saarbr\"ucken, Germany \\
\texttt{andres@mpi-inf.mpg.de}
}

\nipsfinalcopy % Uncomment for camera-ready version

\begin{document}

\maketitle

\begin{abstract}
Estimating a constrained relation is a fundamental problem in machine learning.
Special cases are
\emph{classification} (the problem of estimating a map from a set of to-be-classified elements to a set of labels),
\emph{clustering} (the problem of estimating an equivalence relation on a set) 
and \emph{ranking} (the problem of estimating a linear order on a set).
We contribute a family of probability measures on the set of all relations between two finite, non-empty sets,
which offers a joint abstraction of multi-label classification, correlation clustering and ranking by linear ordering.
Estimating (learning) a maximally probable measure, given (a training set of) related and unrelated pairs, is a convex optimization problem.
Estimating (inferring) a maximally probable relation, given a measure, is a $01$-linear program.
It is solved in linear time for maps.
It is NP-hard for equivalence relations and linear orders.
Practical solutions for all three cases are shown in experiments with real data.
Finally, estimating a maximally probable measure and relation jointly is posed as a mixed-integer nonlinear program.
This formulation suggests a mathematical programming approach to semi-supervised learning.
%
% {\small Keywords: \emph{Clustering, Combinatorial Optimization, Convex Optimization, Graphical Models, Mathematical Programming, Mixed-Integer Nonlinear Programming, Ranking}}
\end{abstract}

\section{Introduction}
\enlargethispage{2ex} % ??? tweaking page break
Given finite, non-empty sets, $A$ and $B$, equal or unequal, 
the problem of estimating a relation between $A$ and $B$ is to decide, 
for every $a \in A$ and every $b \in B$, 
whether or not the pair $ab$ is related.
\emph{Classification}, for instance, is the problem of estimating a map from a set $A$ of to-be-classified elements to a set $B$ of labels by choosing, for every $a \in A$, precisely one label $b \in B$.
\emph{Clustering} is the problem of estimating an equivalence relation on a set $A$ 
by deciding, for every $a,a' \in A$, whether or not $a$ and $a'$ are in the same cluster.
\emph{Ranking} is the problem of estimating a linear order on a set $A$ 
by deciding, for every $a,a' \in A$, whether or not $a$ is less than or equal to $a'$.
In none of these three examples are the decisions pairwise independent:
If the label of $a \in A$ is $b \in B$, it cannot be $b' \in B \setminus \{b\}$.
If $a$ and $a'$ are in the same cluster, and $a'$ and $a''$ are in the same cluster, $a$ and $a''$ cannot be in distinct clusters. 
If $a$ is less than $a'$, $a'$ cannot be less than $a$, etc.
Constraining the set of feasible relations to maps, equivalence relations and linear orders, resp., introduces dependencies.

We define a family of probability measures on the set of all relations between two finite, non-empty sets
such that the relatedness of any pair $ab$ and the relatedness of any pair $a'b' \not= ab$ are independent, albeit with the possibility of being conditionally dependent, given a constrained set of feasible relations.
With respect to this family of probability measures,
we study the problem of estimating (learning) a maximally probable measure, 
given (a training set of) related and unrelated pairs, 
as well as the problem of estimating (inferring) a maximally probable relation, given a measure.
Solutions for classification, clustering and ranking are shown in experiments with real data.
Finally, we state the problem of estimating the measure and relation jointly, for any constrained set of feasible relations, as a mixed-integer nonlinear programming problem (MINLP).
This formulation suggests a mathematical programming approach to semi-supervised learning.
Proofs are deferred to 
Appendix~\ref{section:appendix:proofs}.

\section{Related Work}
\label{section:related-work}

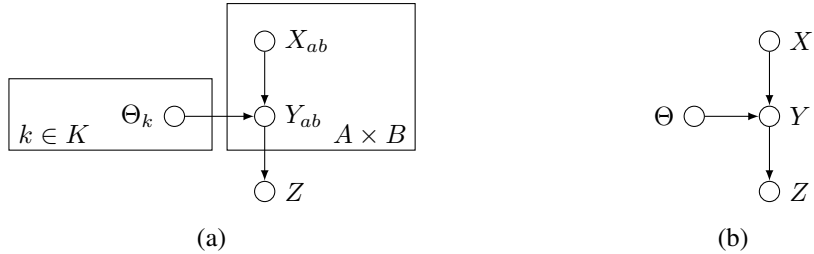
\begin{figure}
\begin{minipage}[b]{0.49\linewidth}
    \centering
    \begin{tikzpicture}
        \node[style=node,label=right:$X_{ab}$] (f) at (1, 2) {};
        \node[style=node,label=right:$Y_{ab}$] (x) at (1, 1) {};
        \node[style=node,label=right:$Z$] (s) at (1, 0) {};
        \node[style=node,label=left:$\Theta_k$] (t) at (-0.2, 1) {};
        \draw (0.5, 0.55) rectangle (3, 2.5);
        \draw (-2.4, 0.55) rectangle (0.3, 1.5);
        \node at (2.4, 0.75) {$A \times B$};
        \node at (-1.8, 0.76) {$k \in K$};
        \path[style=edge] (t)--(x) {};
        \path[style=edge] (f)--(x) {};
        \path[style=edge] (x)--(s) {};
    \end{tikzpicture}\\
    \subcaption{}
    \label{figure:bayesian-model-a}
\end{minipage}
\begin{minipage}[b]{0.49\linewidth}
    \centering
    \begin{tikzpicture}
        \node[style=node,label=right:$X$] (x) at (1, 2) {};
        \node[style=node,label=right:$Y$] (y) at (1, 1) {};
        \node[style=node,label=right:$Z$] (z) at (1, 0) {};
        \node[style=node,label=left:$\Theta$] (t) at (0, 1) {};
        \path[style=edge] (x)--(y) {};
        \path[style=edge] (y)--(z) {};
        \path[style=edge] (t)--(y) {};
    \end{tikzpicture}\\
    \subcaption{}
    \label{figure:bayesian-model-b}
\end{minipage}
\caption{Bayesian models of probability measures. 
a) The model of probability measures on binary relations we consider, in plate notation. 
b) A more general model of probability measures on subsets.}
\label{figure:bayesian-model}
\end{figure}

Estimating a constrained relation is a special case of \emph{structured output prediction} 
\cite{bakir-2007},
the problem of estimating, for a set $S$
and a set $z \subseteq 2^S$ of feasible subsets of $S$,
from observed data $x$,
one feasible subset $y \in z$ of $S$
so as to maximize a margin,
as in 
\cite{tsochantaridis-2005},
entropy, as in 
\cite{lange-2005},
or a conditional probability of $y$, given $x$ and $z$,
as in this work.
For relations, $S = A \times B$ with $A \not= \emptyset$ and $B \not= \emptyset$.

The Bayesian model of probability measures on the set of all relations between two sets we consider
(Fig.~\ref{figure:bayesian-model-a})
is more specific than the general Bayesian model for structured output prediction
(Fig.~\ref{figure:bayesian-model-b}).
Firstly, we assume that the relatedness of a pair $ab$ depends only on one observation, $x_{ab}$, associated with the pair $ab$. We consider no observations associated with multiple pairs.
Secondly, we assume that the relatedness of any pair $ab$ and the relatedness of any pair $a'b' \not= ab$ are independent, albeit with the possibility of being conditionally dependent, given a constrained set $z$ of feasible relations.

One probabilistically principled way of estimating a constrained subset (such as a relation) from observed data is by maximizing entropy 
\cite{lange-2005}.
The maximum probability estimation we perform is different.
It is invariant under transformations of the probability measure that preserve the optimum (possibly a disadvantage), and it does not require sampling.

The problem of estimating a maximally probable \emph{equivalence relation} with respect to the probability measure we consider is known in discrete mathematics as the Set Partition Problem
\cite{chopra-1993}
and in machine learning as correlation clustering
\cite{bansal-2004,demaine-2006}.
The state of the art in solving this NP-hard problem is by branch-and-cut, 
exploiting properties of the Set Partition Polytope
\cite{deza-1997}.
Correlation clustering differs from clustering based on (non-negative) distances.
In correlation clustering, all partitions are, a priori, equally probable.
In distance-based clustering, the prior probability of partitions is typically different from the equipartition (otherwise, the trivial solution of one-elementary clusters would be optimal).
Parameter learning for distance-based clustering is discussed comprehensively in
\cite{xing-2003}.
We discuss parameter learning for equivalence relations and thus, correlation clustering.
Closely related to equivalence relations are multicuts
\cite{chopra-1993};
for a complete graph, the complements of the multicuts are the equivalence relations on the node set.
Multicuts are used, for instance, in image segmentation
\cite{andres-2011,andres-2012,kappes-2013,kim-2014}.
The probability measure on a set of multicuts defined in
\cite{andres-2012}
is a special case of the probability measure we discuss here.

The problem of estimating a maximally probable \emph{linear order} with respect to the probability measure we consider is known as the Linear Ordering Problem.
The state of the art in solving this NP-hard problem is by branch-and-cut, 
exploiting properties of the Linear Ordering Polytope.
The problem and polytope are discussed comprehensively in 
\cite{marti-2011},
along with exact algorithms, approximations and heuristics.
Solutions of the Linear Ordering problem are of interest in machine learning, 
for instance, to predict the order of words in sentences
\cite{tromble-2009}.
Our experiments in 
Section~\ref{section:exp:order}
are inspired by the experiments in
\cite{tromble-2009}.
Unlike in \cite{tromble-2009},
we do not use any linguistic features and assess solutions of the Linear Ordering Problem explicitly.

We concentrate on feature vectors in $\{0,1\}^K$, for a finite index set $K$.
This is w.l.o.g.~on a finite state computer
and has the advantage that every probability measure on the feature space has a (unique) multi-linear polynomial form
\cite{boros-2002}.
We approximate this form by randomized multi-linear polynomial lifting,
building on the approximation of polynomial kernels proposed in
\cite{pham2013fast}.
This modeling of probability measures by \emph{linear} approximations of multi-\emph{linear} polynomial forms
is in stark contrast to the families of \emph{nonlinear}, nonconvex functions modeled by (deep) neural networks.

\newpage
\section{Probability Measures on a Set of Binary Relations}

For any finite, non-empty sets, $A$ and $B$, equal or unequal,
we define the probability of any relation $y' \in 2^{A \times B}$ 
between these sets with respect to 
(i) a set $z' \subseteq 2^{A \times B}$ called the set of \emph{feasible relations}
(ii) a finite index set $J$ and, for every $ab \in A \times B$, 
an $x_{ab} \in \{0,1\}^J$ called the \emph{feature vector} of the pair $ab$
(iii) a finite index set $K$ and a $\theta \in \mathbb{R}^K$ called a
\emph{parameter vector}.
The probability measure is defined with respect to a Bayesian model in four random variables, 
$X$, $Y$, $Z$, and $\Theta$.
The model is depicted in 
Fig.~\ref{figure:bayesian-model-a}.
The random variables and conditional probability measures are defined below.
\begin{itemize}
\item For any $ab \in A \times B$,
a realization of the random variable $X_{ab}$ is
a (feature) vector $x_{ab} \in \{0,1\}^J$.
Thus, a realization of the random variable $X$ is a map
$x: A \times B \to \{0,1\}^J$
from pairs $ab$ to their respective feature vector $x_{ab}$.
\item For any $ab \in A \times B$,
a realization of the random variable $Y_{ab}$ is a $y_{ab} \in \{0,1\}$.
Hence, a realization of the random variable $Y$ is the characteristic vector
$y \in \{0,1\}^{A \times B}$ of a relation between $A$ and $B$, 
namely the relation $y' := \{ab \in A \times B\ |\ y_{ab} = 1\}$.
\item A realization of the random variable $Z$ is a set 
$z \subseteq \{0,1\}^{A \times B}$ of characteristic vectors.
It defines a set $z' \subseteq 2^{A \times B}$ of feasible relations,
namely those relations $y'$ whose characteristic vector $y$ is an element of $z$.
\item A realization of the random variable $\Theta$ is a (parameter) vector $\theta \in \mathbb{R}^K$.
\end{itemize}
From the conditional independence assumptions enforced by the Bayesian model
(Fig.~\ref{figure:bayesian-model-a})
follows that a probability measure of the conditional probability of a relation $y'$ and model parameters $\theta$, given features $x$ of all pairs, and given a set $z'$ of feasible relations, separates according to
\begin{align}
\mathrm{d}p_{Y, \Theta | X, Z}(y, \theta, x, z)
\ \propto \ 
p_{Z|Y}(z, y)
    \prod_{ab \in A \times B} \!\!\!\! p_{Y_{ab} | X_{ab}, \Theta}(y_{ab}, x_{ab}, \theta) 
    \cdot \prod_{k \in K} p_{\Theta_k}(\theta_k)\ 
    \mathrm{d}\theta_k
\enspace .
\label{eq:prob-model}
\end{align}

We define the likelihood $p_{Z|Y}$ of a relation $y'$ to be positive and equal for all feasible relations and zero for all infeasible relations. That is,
\begin{align}
p_{Z|Y}(z, y)
    & \propto \begin{cases}
        1 & \textnormal{if}\ y \in z \\
        0 & \textnormal{otherwise}
    \end{cases}
    \enspace .
    \label{eq:likelihood}
\end{align}

By defining the likelihood $p_{Y_{ab} | X_{ab}, \Theta}$, we choose a family of measures of the probability of a pair being an element of the unconstrained relation, given its features. 
By defining the prior $p_{\Theta_k}$, we choose a distribution of the parameters of this family.
We consider two alternatives, a logistic model and a Bernoulli model, each with respect to a single (regularization) parameter $\sigma \in \mathbb{R}^+$.

\hfill
\begin{tabular}{@{}llll@{\hspace{8ex}}r@{}}
    & $p_{Y_{ab} | X_{ab}, \Theta}(y_{ab}, x_{ab}, \theta)$
    & $p_{\Theta_k}(\theta_k)$\\[2ex]
Logistic
    & $\displaystyle\left( 1 + 2^{-(2 y_{ab} - 1)\langle \theta, x_{ab} \rangle} \right)^{-1}$
    & $\displaystyle\frac{1}{\sigma \sqrt{2 \pi}} \exp\left(- \frac{\theta_k^2}{2 \sigma^2} \right)$
    & $\theta \in \mathbb{R}^K$
    & \eqnum\label{eq:logistic}\\[3ex]
Bernoulli
    & $\displaystyle\prod_{k \in K} \left( \theta_k^{y_{ab}} (1 - \theta_k)^{1 - y_{ab}} \right)^{{x_{ab}}_k}$
    & $\displaystyle \frac{\Gamma(2\sigma)}{\Gamma^2(\sigma)} \theta_k^{\sigma - 1} (1-\theta_k)^{\sigma - 1}$
    & $\theta \in (0,1)^K$
    & \eqnum\label{eq:bernoulli}
\end{tabular}

Our assumption of a \emph{linear} logistic form is without loss of generality, as we show in 
% the supplement.
Appendix~\ref{section:multilinear-polynomial-lifting}.
The Bernoulli model is defined for the special case in which each pair $ab$ is an element of one of finitely many classes, characterized by precisely one non-zero entry of the feature vector $x_{ab}$, and the probability of the pair being an element of the unconstrained relation depends only on its class.

\paragraph{Constraints as Evidence}
Any property of finite relations can be enforced by introducing \emph{evidence},
more precisely, by fixing the random variable $Z$ to the proper subset $z \subset \{0,1\}^{A \times B}$ of precisely the characteristic vectors of those relations $z' \subset 2^{A \times B}$ that exhibit the property.
Two examples are given below.
Firstly, the property that a particular pair $ab \in A \times B$ be an element of the relation and that a different pair $a'b' \in A \times B$ not be an element of the relation is introduced by $z$ defined as the set of all $x \in \{0,1\}^{A \times B}$ such that $x_{ab} = 1$ and $x_{a'b'} = 0$.
Secondly, the property that the relation be a \emph{map} from $A$ to $B$
is introduced by $z$ defined as the set of all $x \in \{0,1\}^{A \times B}$ such that $\forall a \in A: \sum_{b \in B} x_{ab} = 1$.
More examples are given in
Section~\ref{section:special-cases}.

\section{Maximum Probability Estimation}
\subsection{Logistic Model}
\begin{lemma}
\label{lemma:logistic}
$(\hat\theta, \hat y)$ maximizes $p_{Y, \Theta | X, Z}$ defined by 
\eqref{eq:prob-model}, \eqref{eq:likelihood} and \eqref{eq:logistic}
if and only if
$(\hat\theta, \hat y)$ is a solution of the mixed-integer nonlinear program written below. 
Its continuous relaxation need not be convex.
\begin{align}
\min_{y \in z, \theta \in \mathbb{R}^K}\ 
D_x(\theta, y) + R_\sigma(\theta)
\label{eq:main-problem-logistic}
\end{align}
\vspace{-2ex}
\begin{align}
D_x(\theta, y) & = 
\sum_{ab \in A \times B} \left(
    - \langle \theta, x_{ab} \rangle y_{ab}
    + \log_2 \left(
        1 + 2^{\langle \theta, x_{ab} \rangle}
    \right)
\right) 
\label{eq:main-problem-logistic-d}\\
R_\sigma(\theta) & = \frac{\log_2 e}{2 \sigma^2} \| \theta \|_2^2
\label{eq:main-problem-logistic-r}
\end{align}
\end{lemma}

If the relation $y'$ is fixed to some $\hat y'$ (defined, for instance, by training data),
\eqref{eq:main-problem-logistic} 
specializes to the problem of estimating (learning) maximally probable model parameters.
This convex problem, stated below, is well-known as \emph{logistic regression}. It can be solved using convex optimization techniques which have been implemented in mature and numerically stable open source software, notably 
\cite{fan-2008}.
\begin{align}
\min_{\theta \in \mathbb{R}^K}\ 
D_x(\theta, \hat y) + R_\sigma(\theta)
\label{eq:learning}
\end{align}

If the model parameters $\theta$ are fixed to some $\hat\theta$ (learned, for instance, from training data, as described above),
\eqref{eq:main-problem-logistic} 
specializes to the problem of estimating (inferring) a maximally probable relation.
The computational complexity of this $01$-linear program, stated below, depends on the set $z$ of feasible relations. Three special cases are discussed in 
Section~\ref{section:special-cases}.
\begin{align}
\min_{y \in z} \quad -\hspace{-1ex}\sum_{ab \in A \times B} \hspace{-1ex} \langle \hat\theta, x_{ab} \rangle  y_{ab}
\label{eq:inference}
\end{align}

\begin{lemma}
\label{lemma:logistic-bound}
\vspace{1ex}
$0 < \inf\limits_{\theta, y} D_x(\theta, y) \leq |A| |B|$.
\end{lemma}

\subsection{Bernoulli Model}
\begin{lemma}
\label{lemma:bernoulli}
$(\hat\theta, \hat y)$ maximizes $p_{Y, \Theta | X, Z}$ defined by
\eqref{eq:prob-model}, \eqref{eq:likelihood} and \eqref{eq:bernoulli}
if and only if
$(\hat\theta, \hat y)$ is a solution of the mixed-integer nonlinear program written below.
Its continuous relaxation need not be convex.
\begin{align}
\min_{y \in z, \theta \in (0,1)^J}\ 
D_x(\theta, y) + R_\sigma(\theta)
\label{eq:main-problem-bernoulli}
\end{align}
\vspace{-2ex}
\begin{align}
D_x(\theta, y) & = 
\sum_{ab \in A \times B} \left(
    \left( \sum_{j \in J} (x_{ab})_j \log_2 \frac{1 - \theta_j}{\theta_j} \right) y_{ab}
    - \sum_{j \in J} (x_{ab})_j \log_2 (1 - \theta_j)
\right) 
\label{eq:main-problem-bernoulli-d}\\
R_\sigma(\theta) & = (1 - \sigma) \sum_{j \in J} \log_2 \theta_j (1 - \theta_j)
\label{eq:main-problem-bernoulli-r}
\end{align}
\end{lemma}

The problem of estimating (learning) an optimal $\hat\theta$ for a fixed $\hat y$
has the well-known and unique closed-form solution stated below which can be found in linear time. 
For every $j \in J$:
\begin{align}
\hat\theta_j = \frac{m_j^+ + (\sigma - 1)}{m_j^+ + m_j^- + 2(\sigma - 1)}
\quad m_j^+ := \sum_{ab \in A \times B} \hspace{-2ex} (x_{ab})_j \hat y_{ab}
\quad m_j^- := \sum_{ab \in A \times B} \hspace{-2ex} (x_{ab})_j (1 - \hat y_{ab})
\label{eq:learning-bernoulli}
\end{align}

The problem of estimating (inferring) optimal parameters $\hat y$ for a fixed $\hat\theta$ is a $01$-linear program of the same form as 
\eqref{eq:inference},
albeit with different coefficients in the objective function:
\begin{align}
\min_{y \in z} 
\quad 
\sum_{ab \in A \times B}
\left( \sum_{j \in J} (x_{ab})_j \log_2 \frac{1 - \hat\theta_j}{\hat\theta_j} \right) y_{ab}
\label{eq:inference-bernoulli}
\end{align}

\section{Special Cases}
\label{section:special-cases}

\subsection{Maps (Classification)}
\label{section:map}
Classification is the problem of estimating a \emph{map} 
from a finite, non-empty set $A$ of to-be-classified elements 
to a finite, non-empty set $B$ of labels.
A map from $A$ to $B$ is a relation $y' \in 2^{A \times B}$ 
that exhibits the properties which are stated below, 
firstly, in terms of first-order logic and, 
secondly, as constraints on the characteristic vector $y$ of $y'$, 
in terms of integer arithmetic.

\begin{tabular}{@{}llllr@{}}
& & First-Order Logic & Integer Arithmetic\\[1ex]
Existence of images
    & $\forall a \in A:$
    & $\exists b \in B (ab \in y')$
    & $1 \leq \sum_{b \in B} y_{ab}$ 
    & \eqnum\label{eq:map-constr-1} \\
Uniqueness of images
    & $\forall a \in A\ \forall \{b,b'\} \in {B \choose 2}:$
    & $ab \notin y' \vee ab' \notin y'$
    & $y_{ab} + y_{ab'} \leq 1$
    & \eqnum\label{eq:map-constr-2}
\end{tabular}

Obviously, classification is a special case of the problem of estimating a constrained relation. 
In order to establish one-versus-rest classification as a special case
(in Appendix~\ref{section:one-versus-rest}), 
we consider not a feature vector for every pair $ab \in A \times B$ but, instead, a feature vector for every element $a \in A$.
Moreover, we constrain the family of probability measures such that the learning problem separates into a set of independent optimization problems, one for each label.

\subsection{Equivalence Relations (Clustering)}
\label{section:eqr}
Clustering is the problem of estimating a \emph{partition} of a finite, non-empty set $A$.
A partition is a set of non-empty, pairwise disjoint subsets of $A$ whose union is $A$.
The set of all partitions of $A$ is characterized by the set of all equivalence relations on $A$.
For every partition $P \subseteq 2^A$ of $A$, the corresponding equivalence relation $y' \in 2^{A \times A}$ consists of precisely those pairs in $A$ whose elements belong to the same set in the partition. That is $\forall aa' \in A \times A: aa' \in y' \Leftrightarrow \exists S \in P: a \in S \wedge a' \in S$.
Therefore, clustering can be stated equivalently as the problem of estimating an equivalence relation $y' \in 2^{A \times A}$ on $A$.
Equivalence relations are, by definition, reflexive, symmetric and transitive.

\begin{tabular}{@{}l@{\hspace{4.2ex}}l@{\hspace{4.2ex}}l@{\hspace{4.2ex}}lr@{}}
& & First-Order Logic & Integer Arithmetic\\[1ex]
Reflexivity
    & $\forall a \in A:$
    & $aa \in y'$
    & $y_{aa} = 1$ 
    & \eqnum\label{eq:eqr-constr-1}\\
Symmetry 
    & $\forall \{a,a'\} \in {A \choose 2}:$
    & $aa' \in y' \Rightarrow a'a \in y'$
    & $y_{aa'} = y_{a'a}$
    & \eqnum\label{eq:eqr-constr-2}\\
Transitivity
    & $\forall \{a,a',a''\} \in {A \choose 3}:$
    & $aa' \in y' \wedge a'a'' \in y'$
    & $y_{aa'} + y_{a'a''} - 1 \leq y_{aa''}$\\
    & 
    & $\quad \Rightarrow aa'' \in y'$ 
    &
    & \eqnum\label{eq:eqr-constr-3}
\end{tabular}

For equivalence relations, the learning problem is of the general form 
\eqref{eq:learning}.
The inference problems
\eqref{eq:inference} and \eqref{eq:inference-bernoulli},
with the feasible set $z$ defined as the set of those $y \in \{0,1\}^{A \times A}$ that satisfy 
\eqref{eq:eqr-constr-1}--\eqref{eq:eqr-constr-3},
are instances of the NP-hard Set Partition Problem
\cite{chopra-1993}, 
known in machine learning as correlation clustering
\cite{bansal-2004,demaine-2006}.

The state of the art in solving this problem (exactly) is by branch-and-cut, 
exploiting properties of the Set Partition Polytope
\cite{deza-1997}.
Feasible solutions of large and hard instances can be found using heuristics, notably the Kernighan-Lin Algorithm
\cite{kernighan-1970}
that terminates in time $O(|A|^2 \log |A|)$.

\subsection{Linear Orders (Ranking)}
Ranking is the problem of estimating a \emph{linear order} on a finite, non-empty set $A$,
that is, a relation $y' \in 2^{A \times A}$ that is reflexive
\eqref{eq:eqr-constr-1}, transitive \eqref{eq:eqr-constr-3}, antisymmetric and total.

\begin{tabular}{@{}l@{\hspace{5ex}}l@{\hspace{5ex}}l@{\hspace{5ex}}l@{\hspace{8ex}}r@{}}
& & First-Order Logic & Integer Arithmetic\\[1ex]
Antisymmetry
    & $\forall \{a,a'\} \in {A \choose 2}:$
    & $aa' \notin y' \vee a'a \notin y'$
    & $y_{aa'} + y_{a'a} \leq 1$
    & \eqnum\label{eq:antisymmetry}\\
Totality
    & $\forall \{a,a'\} \in {A \choose 2}:$
    & $aa' \in y' \vee a'a \in y'$
    & $1 \leq y_{aa'} + y_{a'a}$
    & \eqnum\label{eq:totality}
\end{tabular}

For linear orders, the learning problem is of the general form 
\eqref{eq:learning}.
The inference problems
\eqref{eq:inference} and \eqref{eq:inference-bernoulli},
with the feasible set $z$ defined as the set of those $y \in \{0,1\}^{A \times A}$ that satisfy 
\eqref{eq:eqr-constr-1}, \eqref{eq:eqr-constr-3}, \eqref{eq:antisymmetry} and \eqref{eq:totality},
are instances of the NP-hard Linear Ordering Problem
\cite{marti-2011}.

The state of the art in solving this problem (exactly) is by branch-and-cut,
exploiting properties of the Linear Ordering Polytope,
cf.~\cite{marti-2011}, Chapter 6.
Feasible solutions of large and hard instances can be found using heuristics,
cf.~\cite{marti-2011}, Chapter 2.

\newpage
\section{Experiments}

The formalism introduced above is used to estimate maps, equivalence relations and linear orders from real data.
All figures reported in this section result from computations on one core of an 
Intel Xeon E5-2660 CPU operating at $2.20$~GHz.
Absolute computation times are shown in 
Appendix~\ref{section:appendix:experiments}.
% The data and C++ code are freely and publicly available
% \cite{lecun-1998,andres-2014}.

\subsection{Maps (Classification)}
\label{section:exp:map}

\begin{figure}
\centering
% This file was created by matlab2tikz v0.4.7 (commit faeee56ab2febcb443644b48a5a8422f5522d43f) running on MATLAB 8.3.
% Copyright (c) 2008--2014, Nico Schlömer <nico.schloemer@gmail.com>
% All rights reserved.
% Minimal pgfplots version: 1.3
% 
% The latest updates can be retrieved from
%   http://www.mathworks.com/matlabcentral/fileexchange/22022-matlab2tikz
% where you can also make suggestions and rate matlab2tikz.
% 
\begin{tikzpicture}

\begin{axis}[%
width=0.35\columnwidth,
height=0.35\columnwidth,
unbounded coords=jump,
scale only axis,
xmode=log,
xmin=1,
xmax=1073741824,
xminorticks=true,
xlabel={$\sigma{}^{\text{-2}}$},
ymin=-0.005,
ymax=0.155,
ytick={   0, 0.05,  0.1, 0.15},
ylabel={Relative error},
legend style={at={(1.03,1)},anchor=north west,fill=none,draw=none,legend cell align=left},
y tick label style={/pgf/number format/.cd, fixed, fixed zerofill, precision=2, /tikz/.cd}
]
\addplot [color=white!70!red,solid]
  table[row sep=crcr]{2	0.0184199996292591\\
4	0.0227400008589029\\
8	0.0273800007998943\\
16	0.034120000898838\\
32	0.0413199998438358\\
64	0.0508599989116192\\
128	0.061039999127388\\
256	0.0713799968361855\\
512	0.0821999981999397\\
1024	0.0927800014615059\\
2048	0.102399997413158\\
4096	0.113200001418591\\
8192	0.122639998793602\\
16384	0.132860004901886\\
32768	0.146660000085831\\
65536	0.161459997296333\\
131072	0.17858000099659\\
262144	0.197339996695518\\
524288	0.215259999036789\\
1048576	nan\\
2097152	nan\\
4194304	nan\\
8388608	nan\\
16777216	nan\\
33554432	nan\\
67108864	nan\\
134217728	nan\\
268435456	nan\\
536870912	nan\\
};
\addlegendentry{d=1, Training};

\addplot [color=red,solid]
  table[row sep=crcr]{2	0.106600001454353\\
4	0.101800002157688\\
8	0.0964000001549721\\
16	0.0920000001788139\\
32	0.0879999995231628\\
64	0.0860000029206276\\
128	0.0852999985218048\\
256	0.0859000012278557\\
512	0.088699996471405\\
1024	0.0921000018715858\\
2048	0.09740000218153\\
4096	0.104800000786781\\
8192	0.112899996340275\\
16384	0.123400002717972\\
32768	0.13570000231266\\
65536	0.14920000731945\\
131072	0.164100006222725\\
262144	0.181999996304512\\
524288	0.201100006699562\\
1048576	nan\\
2097152	nan\\
4194304	nan\\
8388608	nan\\
16777216	nan\\
33554432	nan\\
67108864	nan\\
134217728	nan\\
268435456	nan\\
536870912	nan\\
};
\addlegendentry{d=1, Test};

\addplot [color=white!70!green,dashed]
  table[row sep=crcr]{2	1.99999994947575e-05\\
4	1.99999994947575e-05\\
8	1.99999994947575e-05\\
16	1.99999994947575e-05\\
32	1.99999994947575e-05\\
64	1.99999994947575e-05\\
128	1.99999994947575e-05\\
256	1.99999994947575e-05\\
512	1.99999994947575e-05\\
1024	0\\
2048	3.9999998989515e-05\\
4096	0.00015999999595806\\
8192	0.00039999998989515\\
16384	0.000880000006873161\\
32768	0.00371999992057681\\
65536	0.00953999999910593\\
131072	0.0177200008183718\\
262144	0.0282600000500679\\
524288	0.0404799990355968\\
1048576	0.0531199984252453\\
2097152	0.0661400035023689\\
4194304	0.0794000029563904\\
8388608	0.0935600027441978\\
16777216	0.109719999134541\\
33554432	0.127000004053116\\
67108864	0.147780001163483\\
134217728	0.171200007200241\\
268435456	0.196559995412827\\
536870912	0.224639996886253\\
};
\addlegendentry{d=2, m=8192, Training};

\addplot [color=green,dashed]
  table[row sep=crcr]{2	0.0370000004768372\\
4	0.0377000011503696\\
8	0.0373000018298626\\
16	0.0379999987781048\\
32	0.0381000004708767\\
64	0.0373000018298626\\
128	0.0375000014901161\\
256	0.0366999991238117\\
512	0.0370000004768372\\
1024	0.0366000011563301\\
2048	0.0364000014960766\\
4096	0.0357000008225441\\
8192	0.0355999991297722\\
16384	0.0355000011622906\\
32768	0.0364999994635582\\
65536	0.0390999987721443\\
131072	0.0414999984204769\\
262144	0.0467999987304211\\
524288	0.0513000003993511\\
1048576	0.0568999983370304\\
2097152	0.0662999972701073\\
4194304	0.0768999978899956\\
8388608	0.0879999995231628\\
16777216	0.101300001144409\\
33554432	0.116899996995926\\
67108864	0.138300001621246\\
134217728	0.162499994039536\\
268435456	0.1841000020504\\
536870912	0.212899997830391\\
};
\addlegendentry{d=2, m=8192, Test};

\addplot [color=white!70!green,solid]
  table[row sep=crcr]{2	1.99999994947575e-05\\
4	1.99999994947575e-05\\
8	1.99999994947575e-05\\
16	1.99999994947575e-05\\
32	1.99999994947575e-05\\
64	1.99999994947575e-05\\
128	1.99999994947575e-05\\
256	1.99999994947575e-05\\
512	1.99999994947575e-05\\
1024	1.99999994947575e-05\\
2048	1.99999994947575e-05\\
4096	0\\
8192	0.000180000002728775\\
16384	0.00095999997574836\\
32768	0.00225999997928739\\
65536	0.00786000024527311\\
131072	0.0157399997115135\\
262144	0.0258200000971556\\
524288	0.0381599999964237\\
1048576	0.0510599985718727\\
2097152	0.0639799982309341\\
4194304	0.0782200023531914\\
8388608	0.0935600027441978\\
16777216	0.109399996697903\\
33554432	0.128340005874634\\
67108864	0.147919997572899\\
134217728	0.170980006456375\\
268435456	0.19684000313282\\
536870912	0.224600002169609\\
};
\addlegendentry{d=2, m=16384, Training};

\addplot [color=green,solid]
  table[row sep=crcr]{2	0.0324999988079071\\
4	0.0324999988079071\\
8	0.0324999988079071\\
16	0.0324000008404255\\
32	0.032200001180172\\
64	0.0324000008404255\\
128	0.0322999991476536\\
256	0.032200001180172\\
512	0.0322999991476536\\
1024	0.0318000018596649\\
2048	0.0313999988138676\\
4096	0.031700000166893\\
8192	0.0315000005066395\\
16384	0.031700000166893\\
32768	0.0329000018537045\\
65536	0.0348000004887581\\
131072	0.0381000004708767\\
262144	0.0434000007808208\\
524288	0.0494999997317791\\
1048576	0.0551000013947487\\
2097152	0.0640999972820282\\
4194304	0.0758000016212463\\
8388608	0.0876000002026558\\
16777216	0.100100003182888\\
33554432	0.116800002753735\\
67108864	0.13740000128746\\
134217728	0.160400003194809\\
268435456	0.185100004076958\\
536870912	0.211300000548363\\
};
\addlegendentry{d=2, m=16384, Test};

\addplot [color=white!70!blue,dashed]
  table[row sep=crcr]{2	0\\
4	0\\
8	0\\
16	0\\
32	0\\
64	0\\
128	0\\
256	0\\
512	0\\
1024	0\\
2048	0\\
4096	0\\
8192	0\\
16384	0\\
32768	0\\
65536	0\\
131072	0\\
262144	0\\
524288	1.99999994947575e-05\\
1048576	3.9999998989515e-05\\
2097152	0.000199999994947575\\
4194304	0.00028000000747852\\
8388608	0.000500000023748726\\
16777216	0.00173999997787178\\
33554432	0.00553999980911613\\
67108864	0.0118199996650219\\
134217728	0.0204399991780519\\
268435456	0.0306599996984005\\
536870912	0.0418599992990494\\
};
\addlegendentry{d=3, m=8192, Training};

\addplot [color=blue,dashed]
  table[row sep=crcr]{2	0.046000000089407\\
4	0.0456999987363815\\
8	0.046000000089407\\
16	0.046000000089407\\
32	0.046000000089407\\
64	0.0458999983966351\\
128	0.046000000089407\\
256	0.046000000089407\\
512	0.0458999983966351\\
1024	0.0458999983966351\\
2048	0.046000000089407\\
4096	0.046000000089407\\
8192	0.0458999983966351\\
16384	0.0456000007688999\\
32768	0.0456999987363815\\
65536	0.046000000089407\\
131072	0.0458999983966351\\
262144	0.0458999983966351\\
524288	0.046000000089407\\
1048576	0.0454000011086464\\
2097152	0.0458000004291534\\
4194304	0.045099999755621\\
8388608	0.0439999997615814\\
16777216	0.0425999984145164\\
33554432	0.0419999994337559\\
67108864	0.0436000004410744\\
134217728	0.0461000017821789\\
268435456	0.0494999997317791\\
536870912	0.0555000007152557\\
};
\addlegendentry{d=3, m=8192, Test};

\addplot [color=white!70!blue,solid]
  table[row sep=crcr]{2	0\\
4	0\\
8	0\\
16	0\\
32	0\\
64	0\\
128	0\\
256	0\\
512	0\\
1024	0\\
2048	0\\
4096	0\\
8192	0\\
16384	0\\
32768	0\\
65536	0\\
131072	0\\
262144	0\\
524288	1.99999994947575e-05\\
1048576	3.9999998989515e-05\\
2097152	5.99999984842725e-05\\
4194304	9.99999974737875e-05\\
8388608	0.000199999994947575\\
16777216	0.00067999999737367\\
33554432	0.00292000011540949\\
67108864	0.00837999954819679\\
134217728	0.0158600006252527\\
268435456	0.0262000001966953\\
536870912	0.0372000001370907\\
};
\addlegendentry{d=3, m=16384, Training};

\addplot [color=blue,solid]
  table[row sep=crcr]{2	0.0366999991238117\\
4	0.0368000008165836\\
8	0.0368000008165836\\
16	0.0368999987840652\\
32	0.0368000008165836\\
64	0.0370000004768372\\
128	0.0370000004768372\\
256	0.0370000004768372\\
512	0.0368000008165836\\
1024	0.0366999991238117\\
2048	0.0368000008165836\\
4096	0.0368000008165836\\
8192	0.0368000008165836\\
16384	0.0366000011563301\\
32768	0.0364999994635582\\
65536	0.0366000011563301\\
131072	0.0364000014960766\\
262144	0.0368000008165836\\
524288	0.0361000001430511\\
1048576	0.0359999984502792\\
2097152	0.0361000001430511\\
4194304	0.0359000004827976\\
8388608	0.0350000001490116\\
16777216	0.0355000011622906\\
33554432	0.0359999984502792\\
67108864	0.0372000001370907\\
134217728	0.0401999987661839\\
268435456	0.0439999997615814\\
536870912	0.0491999983787537\\
};
\addlegendentry{d=3, m=16384, Test};

\end{axis}
\end{tikzpicture}%
\caption{Classification of images of handwritten digits (MNIST), using the proposed model.}
\label{figure:map-experiment-main}
\end{figure}
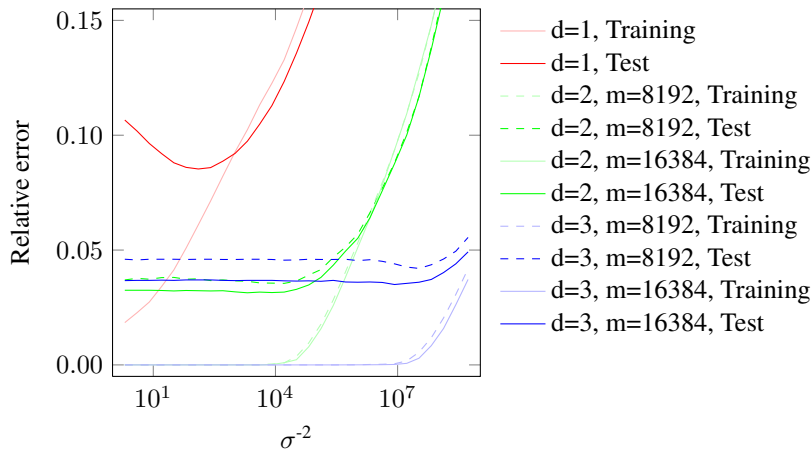

Firstly, we consider the problem of classifying images of handwritten digits of the raw MNIST data set
\cite{lecun-1998},
based on a 6272-dimensional vector of $01$-features (8 bits for each of 28$\cdot$28 pixels).
Multilinear polynomial liftings of the feature space are described in 
% the supplement.
Appendix~\ref{section:multilinear-polynomial-lifting}.

Fig.~\ref{figure:map-experiment-main}
shows fractions of misclassified images.
It can be seen from this figure that the minimal error on the test set is as low as 8.53\% (at $\sigma^{-2} = 2^7$) for a linear function ($d=1$), thanks to the $01$-features.
For an approximation of a multilinear polynomial form of degree $d=2$ by $m=16348$ random features (see 
% supplement 
Appendix~\ref{section:multilinear-polynomial-lifting}
for details), the error drops to 3.14\% (at $\sigma^{-2} = 2^{12}$).
Reducing the number of random features by half increases the error by $0.5\%$.
Approximating a multilinear polynomial form of degree $d=3$ by $m \leq 16348$ random features yields worse results.
The overall best result of 3.14\% misclassified images falls short of the impressive state of the art of 0.21\% defined by deep learning
\cite{wan-2013}
and encourages future work on multilinear polynomial lifting.

\subsection{Equivalence Relations (Clustering)}
\label{section:exp:eqr}

\begin{figure}
\centering
\input{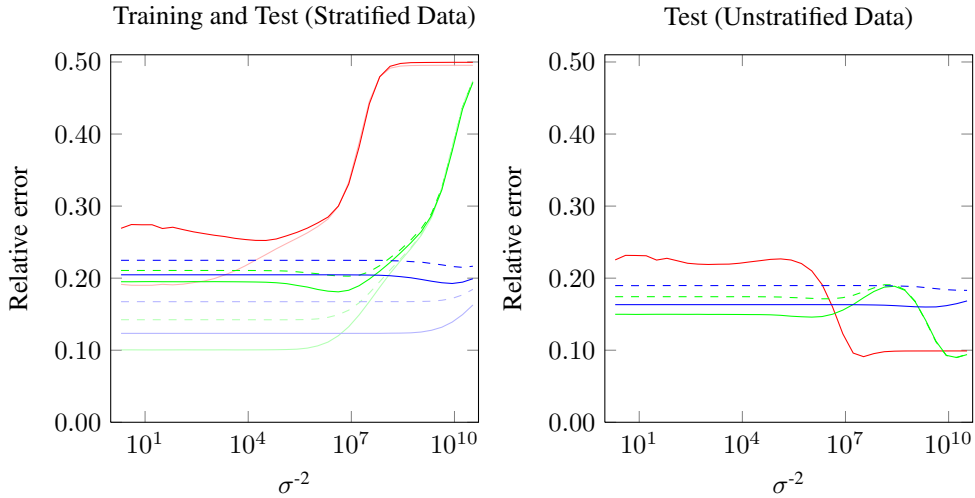}
% This file was created by matlab2tikz v0.4.7 (commit faeee56ab2febcb443644b48a5a8422f5522d43f) running on MATLAB 8.3.
% Copyright (c) 2008--2014, Nico Schlömer <nico.schloemer@gmail.com>
% All rights reserved.
% Minimal pgfplots version: 1.3
% 
% The latest updates can be retrieved from
%   http://www.mathworks.com/matlabcentral/fileexchange/22022-matlab2tikz
% where you can also make suggestions and rate matlab2tikz.
% 
\begin{tikzpicture}

\begin{axis}[%
width=0.35\columnwidth,
height=0.35\columnwidth,
scale only axis,
xmode=log,
xmin=1,
xmax=50000000000,
xminorticks=true,
xlabel={$\sigma{}^{\text{-2}}$},
ymin=0,
ymax=0.51,
ytick={  0, 0.1, 0.2, 0.3, 0.4, 0.5},
ylabel={Relative error},
title={Test (Unstratified Data)},
y tick label style={/pgf/number format/.cd, fixed, fixed zerofill, precision=2, /tikz/.cd}
]
\addplot [color=red,solid,forget plot]
  table[row sep=crcr]{2	0.225311311311311\\
4	0.231623623623624\\
8	0.231453453453453\\
16	0.231069069069069\\
32	0.225081081081081\\
64	0.227561561561562\\
128	0.22412012012012\\
256	0.220652652652653\\
512	0.219469469469469\\
1024	0.218982982982983\\
2048	0.219321321321321\\
4096	0.219479479479479\\
8192	0.220688688688689\\
16384	0.222540540540541\\
32768	0.224252252252252\\
65536	0.22605005005005\\
131072	0.226720720720721\\
262144	0.225179179179179\\
524288	0.220262262262262\\
1048576	0.21003003003003\\
2097152	0.190982982982983\\
4194304	0.160208208208208\\
8388608	0.122688688688689\\
16777216	0.096034034034034\\
33554432	0.091043043043043\\
67108864	0.0953053053053053\\
134217728	0.0979279279279279\\
268435456	0.0986846846846847\\
536870912	0.0989129129129129\\
1073741824	0.098950950950951\\
2147483648	0.098972972972973\\
4294967296	0.098984984984985\\
8589934592	0.099005005005005\\
17179869184	0.099005005005005\\
34359738368	0.099005005005005\\
};
\addplot [color=green,dashed,forget plot]
  table[row sep=crcr]{2	0.174272272272272\\
4	0.174272272272272\\
8	0.1743003003003\\
16	0.174272272272272\\
32	0.174292292292292\\
64	0.174272272272272\\
128	0.17427027027027\\
256	0.17427027027027\\
512	0.174268268268268\\
1024	0.174264264264264\\
2048	0.17425025025025\\
4096	0.174228228228228\\
8192	0.17429029029029\\
16384	0.174234234234234\\
32768	0.174188188188188\\
65536	0.174052052052052\\
131072	0.173781781781782\\
262144	0.173099099099099\\
524288	0.172626626626627\\
1048576	0.171945945945946\\
2097152	0.171347347347347\\
4194304	0.171445445445445\\
8388608	0.172934934934935\\
16777216	0.176182182182182\\
33554432	0.181179179179179\\
67108864	0.186472472472472\\
134217728	0.190878878878879\\
268435456	0.190698698698699\\
536870912	0.183757757757758\\
1073741824	0.16777977977978\\
2147483648	0.140524524524525\\
4294967296	0.111535535535536\\
8589934592	0.0925925925925926\\
17179869184	0.0900620620620621\\
34359738368	0.0942962962962963\\
};
\addplot [color=green,solid,forget plot]
  table[row sep=crcr]{2	0.149861861861862\\
4	0.149847847847848\\
8	0.149683683683684\\
16	0.149781781781782\\
32	0.149861861861862\\
64	0.149857857857858\\
128	0.149843843843844\\
256	0.149747747747748\\
512	0.149837837837838\\
1024	0.149671671671672\\
2048	0.149717717717718\\
4096	0.149641641641642\\
8192	0.149581581581582\\
16384	0.149467467467467\\
32768	0.149411411411411\\
65536	0.149129129129129\\
131072	0.148418418418418\\
262144	0.147203203203203\\
524288	0.146448448448448\\
1048576	0.145927927927928\\
2097152	0.146714714714715\\
4194304	0.149923923923924\\
8388608	0.155667667667668\\
16777216	0.163427427427427\\
33554432	0.172832832832833\\
67108864	0.181507507507508\\
134217728	0.187525525525526\\
268435456	0.188762762762763\\
536870912	0.183099099099099\\
1073741824	0.166980980980981\\
2147483648	0.141577577577578\\
4294967296	0.112474474474474\\
8589934592	0.0928908908908909\\
17179869184	0.0898178178178178\\
34359738368	0.0939219219219219\\
};
\addplot [color=blue,dashed,forget plot]
  table[row sep=crcr]{2	0.1896996996997\\
4	0.189693693693694\\
8	0.189583583583584\\
16	0.189707707707708\\
32	0.189675675675676\\
64	0.1896996996997\\
128	0.18973973973974\\
256	0.189745745745746\\
512	0.189663663663664\\
1024	0.189653653653654\\
2048	0.1896996996997\\
4096	0.189777777777778\\
8192	0.189701701701702\\
16384	0.189645645645646\\
32768	0.189631631631632\\
65536	0.18974974974975\\
131072	0.189695695695696\\
262144	0.189703703703704\\
524288	0.189713713713714\\
1048576	0.189733733733734\\
2097152	0.189697697697698\\
4194304	0.18970970970971\\
8388608	0.189721721721722\\
16777216	0.189771771771772\\
33554432	0.189667667667668\\
67108864	0.189583583583584\\
134217728	0.189489489489489\\
268435456	0.189321321321321\\
536870912	0.188828828828829\\
1073741824	0.188298298298298\\
2147483648	0.186866866866867\\
4294967296	0.185425425425425\\
8589934592	0.183995995995996\\
17179869184	0.183287287287287\\
34359738368	0.183033033033033\\
};
\addplot [color=blue,solid,forget plot]
  table[row sep=crcr]{2	0.163233233233233\\
4	0.163233233233233\\
8	0.163235235235235\\
16	0.163235235235235\\
32	0.163235235235235\\
64	0.163259259259259\\
128	0.163233233233233\\
256	0.163235235235235\\
512	0.163235235235235\\
1024	0.163235235235235\\
2048	0.163233233233233\\
4096	0.163233233233233\\
8192	0.163235235235235\\
16384	0.163233233233233\\
32768	0.163235235235235\\
65536	0.163233233233233\\
131072	0.163233233233233\\
262144	0.163231231231231\\
524288	0.163259259259259\\
1048576	0.163245245245245\\
2097152	0.163231231231231\\
4194304	0.163217217217217\\
8388608	0.163177177177177\\
16777216	0.163153153153153\\
33554432	0.162998998998999\\
67108864	0.162804804804805\\
134217728	0.162486486486486\\
268435456	0.161877877877878\\
536870912	0.161091091091091\\
1073741824	0.160276276276276\\
2147483648	0.159951951951952\\
4294967296	0.160098098098098\\
8589934592	0.16180980980981\\
17179869184	0.164584584584585\\
34359738368	0.168388388388388\\
};
\end{axis}
\end{tikzpicture}%
\caption{Classification of pairs of images of handwritten digits (MNIST). 
Colors and line styles have the same meaning as in
Fig.~\ref{figure:map-experiment-main}.}
\label{figure:eqr-experiment-main}
\end{figure}

Next, we consider the problem of clustering sets of images of handwritten digits, including the entire MNIST test set of $10^4$ images.
A training set $\{(x_{aa'}, y_{aa'})\}_{aa' \in T}$ of $|T| = 5 \cdot 10^5$ pairs of images is drawn randomly and without replacement from the MNIST training set, such that it contains as many pairs of images showing the same digit as pairs of images showing distinct digits.
(Results for learning from unstratified data are shown in 
Appendix~\ref{section:appendix:experiments}.)
For every pair $aa' \in T$ of images, $x_{aa'}$ is a 12544-dimensional $01$-vector (defined in 
% the supplement),
Appendix~\ref{section:features-of-pairs}),
and $y_{aa'} = 1$ iff the images show (are labeled with) the same digit. 
Stratified and unstratified test sets of pairs of images are drawn randomly and without replacement from the MNIST test set.
Results for the independent classification of pairs (not a solution of the Set Partition Problem) are shown in 
Fig.~\ref{figure:eqr-experiment-main}.
The fraction of misclassified pairs is 
18.1\% on stratified test data and
15.0\% on unstratified test data,
both at $\sigma^{-2} = 2^{22}$ and for an approximation of a multilinear polynomial form of degree $d=2$ by 16384 random features.

For $\hat\theta$ learned with these parameters,
we infer equivalence relations on random subsets $A$ of the MNIST test set 
by solving the Set Partition Problem, that is,
\eqref{eq:inference} 
with the feasible set $z$ defined as the set of those $y \in \{0,1\}^{A \times A}$ that satisfy 
\eqref{eq:eqr-constr-1}--\eqref{eq:eqr-constr-3}.
For small instances, we use the branch-and-cut loop of the closed-source commercial software IBM ILOG Cplex.
In this loop, we separate the inequalities 
\eqref{eq:eqr-constr-1}--\eqref{eq:eqr-constr-3}.
Beyond these, we resort to the general classes of cuts implemented in Cplex.
%We do not separate odd-wheel inequalities.
For large instances, we initialize our implementation of the Kernighan-Lin Algorithm
with the feasible solution in which a pair $aa' \in A \times A$ is related iff there exists a path from $a$ to $a'$ in the complete graph $K_A$ such that, for all edges $a''a'''$ in the path, $\hat\theta_{a''a'''} > 0$.
% The initial feasible solution is, with high probability, close in Hamming distance to the trivial feasible solution $x = 1$. 
% Using it as an initialization is justified by the fact that initial labels $y_{aa'} = 0$ are persistent.
%
An evaluation of equivalence relations on random subsets $A$ of the MNIST test set in terms of 
the fraction $e_{\textnormal{RI}}$ of misclassified pairs (one minus Rand's index),
the variation of information  
\cite{meila-2003}
and the objective value of the Set Partition Problem
is shown in
Tab.~\ref{table:results-eqr}.
It can be seen form this table that the fixed points $\hat y^{\kl}$ of the Kernighan-Lin Algorithm closely approximate certified optimal solutions $\hat y$.
It can also be seen that the runtime $t$ of the Kernighan-Lin Algorithm, unlike that of our branch-and-cut procedure, is practical for clustering the entire MNIST test set.
Finally, it can be seen that the heuristic feasible solution $\hat y^{\kl}$ of the Set Partition Problem  reduces the fraction of pairs of images classified incorrectly from 
15.0\% (for independent classification) 
or 10\% (for the trivial partition into one-elementary sets)
to $7.11\%$.

\begin{table}
\caption{Comparison of equivalence relations on stratified random subsets $A$ of the MNIST test set}
\label{table:results-eqr}
\begin{center}
\begin{small}
\begin{tabular}{@{}l@{\hspace{1.5ex}}l@{\hspace{1.5ex}}r@{\hspace{1.5ex}}r@{\hspace{1.5ex}}r@{\hspace{1.5ex}}r@{\hspace{1.5ex}}r@{\hspace{1.5ex}}r@{\hspace{1.5ex}}r@{\ }r@{\ }r@{}}
\hline\\[-1ex]

& 
& \multicolumn{1}{@{}l}{$|A|$}
& \multicolumn{1}{@{}l}{${|A| \choose 2}$}
& \multicolumn{1}{@{}l}{$e_{\mathrm{RI}}\ [\%]$}
& \multicolumn{1}{@{}l}{$\vi$ \cite{meila-2003}}
& \multicolumn{1}{@{}l}{Sets}
& \multicolumn{1}{@{}l}{Obj./${|A| \choose 2} \cdot 10^2$} 
& \multicolumn{3}{@{}l}{$t\ [s]$}\\[1ex]

\hline\\[-1ex]

\multirow{9}{*}{$\hat\theta^s$} & \multirow{3}{*}{$\hat y$} & $100$   & $4950$   & $8.81$ $\pm$ $1.65$   & $\bf 1.16$ $\pm$ $0.20$   & $12.5$ $\pm$ $1.5$   & $-16.88$ $\pm$ $2.36$   & $\bf 34.99$   & $\pm$    & $27.86$\\
& & $170$   & $14365$   & $7.23$ $\pm$ $0.87$   & $\bf 1.08$ $\pm$ $0.14$   & $16.6$ $\pm$ $2.5$   & $-16.33$ $\pm$ $1.18$   & $\bf 2777.30$   & $\pm$    & $4532.56$\\
& & $220$   & $24090$   & $7.70$ $\pm$ $0.59$   & $\bf 1.23$ $\pm$ $0.11$   & $19.6$ $\pm$ $2.5$   & $-16.00$ $\pm$ $0.91$   & $\bf 35424.71$   & $\pm$    & $95316.62$\\[1ex]

& \multirow{6}{*}{$\hat y^{\kl}$} & $100$   & $4950$   & $8.69$ $\pm$ $1.37$   & $\bf 1.15$ $\pm$ $0.21$   & $12.3$ $\pm$ $1.9$   & $-16.87$ $\pm$ $2.36$   & $\bf 0.01$   & $\pm$    & $0.00$\\
& & $170$   & $14365$   & $7.36$ $\pm$ $0.80$   & $\bf 1.09$ $\pm$ $0.15$   & $16.5$ $\pm$ $2.3$   & $-16.33$ $\pm$ $1.19$   & $\bf 0.03$   & $\pm$    & $0.01$\\
& & $220$   & $24090$   & $7.71$ $\pm$ $0.60$   & $\bf 1.24$ $\pm$ $0.11$   & $19.7$ $\pm$ $2.5$   & $-16.00$ $\pm$ $0.91$   & $\bf 0.05$   & $\pm$    & $0.01$\\
& & $260$   & $33670$   & $7.59$ $\pm$ $0.85$   & $\bf 1.25$ $\pm$ $0.14$   & $20.1$ $\pm$ $1.8$   & $-15.56$ $\pm$ $1.25$   & $\bf 0.06$   & $\pm$    & $0.01$\\
& & $300$   & $44850$   & $7.83$ $\pm$ $0.79$   & $\bf 1.22$ $\pm$ $0.10$   & $20.9$ $\pm$ $3.0$   & $-16.49$ $\pm$ $0.58$   & $\bf 0.09$   & $\pm$    & $0.01$\\

& & $\bf 10^4$ & $\bf 5 \cdot 10^9$ & \multicolumn{1}{@{}l}{$\bf 7.11$} & \multicolumn{1}{@{}l}{$\bf 1.49$} & \multicolumn{1}{@{}l}{$\bf 129$} & $\bf -16.65$ & $\bf 595.19$\\[1ex]

\hline
\end{tabular}
\end{small}
\end{center}
\end{table}

\subsection{Linear Orders (Ranking)}
\label{section:exp:order}

\begin{figure}
\includegraphics{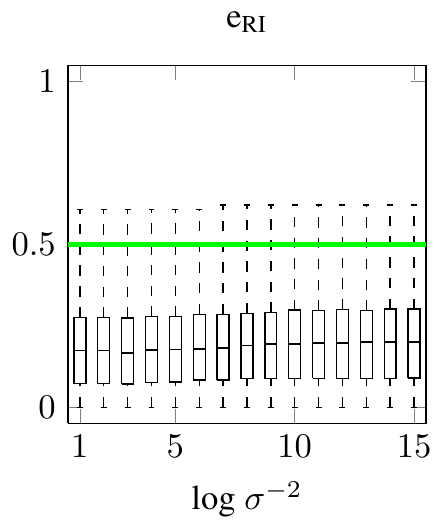}\hfill
\includegraphics{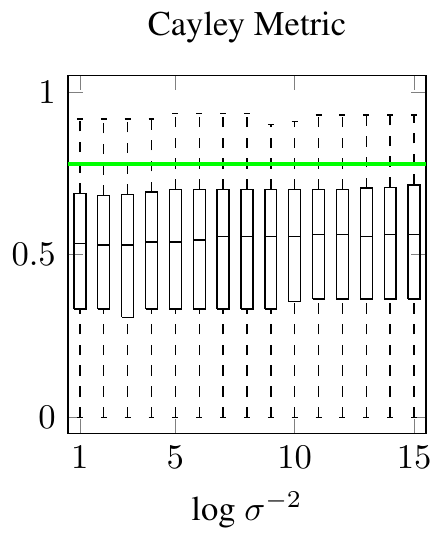}\hfill
\includegraphics{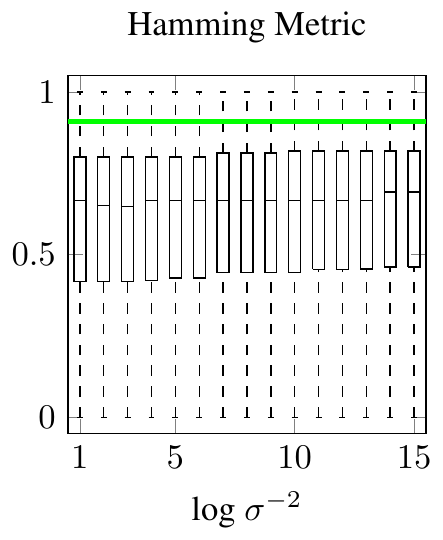}
\caption{Distances between the optimal sentences with respect to the model and the correct sentences.}
\label{figure:metric}
\end{figure}

Finally, we consider the problem of estimating the linear order of words in sentences.
Training data is provided by every well-formed sentence and is therefore abundant.
We estimate, for every pair $jj'$ of words $j$ and $j'$ in a dictionary,
the probability of the word $j$ to occur before the word $j'$ in a sentence.
Our dictionary consists of the 1,000 words most often used in the English Wikipedia.
Our training (test) data consists of 129,389 (10,000) sentences, drawn randomly and without replacement from those sentences in the English Wikipedia that contain only words from the dictionary.
We define $A$ to be the set of all \emph{occurrences} of words in a sentence, as the same word can occur multiple times.
For every pair $aa'$ of occurrences of words, the feature vector $x_{aa'} \in \{0,1\}^J$ is indexed by $J$, the set of all pairs of words in the dictionary.
We define $(x_{aa'})_{jj'} = 1$ iff $a$ is an occurrence of the word $j$ and $a'$ is an occurrence of the word $j'$.
With respect to the Bernoulli model, 
optimal parameters $\hat\theta$ are learned by evaluating the closed form 
\eqref{eq:learning-bernoulli},
which takes less than 10~seconds for the entire training set.
Every sentence of the test set is taken to be an unordered set of words and is permuted randomly for this experiment.
An optimal linear order of words is estimated by solving the Linear Ordering Problem, that is,
\eqref{eq:inference-bernoulli},
with the feasible set $z$ defined as the set of those $y \in \{0,1\}^{A \times B}$ that satisfy 
\eqref{eq:eqr-constr-1}, \eqref{eq:eqr-constr-3}, \eqref{eq:antisymmetry} and \eqref{eq:totality}.
For all instances, we use the branch-and-cut loop of Cplex, separating the inequalities 
\eqref{eq:eqr-constr-1}, \eqref{eq:eqr-constr-3}, \eqref{eq:antisymmetry} and \eqref{eq:totality}
and otherwise resorting to the general classes of cuts implemented in Cplex.

Metric distances between the optimal sentences with respect to the model
and the correct sentences are reported for three different metrics 
\cite{deza1998metrics}
in 
Fig.~\ref{figure:metric}.
For the summary statistics in this figure, the metrics have been normalized appropriately to account for the different lengths of sentences.
Horizontal lines indicate the value the normalized metric would assume for randomly ordered sentences.
It can be seen from this figure that the model is effective in estimating the order of words in sentences
and is not sensitive to the regularization parameter $\sigma$ for the dictionary and training data we used.

\section{Conclusion}
We have defined a family of probability measures on the set of all relations between two finite, non-empty sets
which offers a joint abstraction of multi-label classification, correlation clustering and ranking by linear ordering.
The problem of estimating (learning) a maximally probable measure, given (a training set of) related and unrelated pairs, is a convex optimization problem.
The problem of estimating (inferring) a maximally probable relation, given a measure, is a $01$-linear program which specializes to the NP-hard Set Partition Problem for equivalence relations
and to the NP-hard Linear Ordering Problem for linear orders.
Experiments with real data have shown that maximum probability learning and maximum probability inference are practical for some instances.

In the experiments we conduct, 
the distinction between learning and inference is motivated by the distinction between training data and test data.
It is well-known, however, that a distinction between learning and inference is inappropriate if there is just one data set and partial evidence about a to-be-estimated relation.
With respect to this setting
which falls into the broader research area of \emph{semi-supervised learning},
we have stated the problem of estimating $\theta$ and $x$ jointly 
as the mixed-integer nonlinear programs
\eqref{eq:main-problem-logistic}--\eqref{eq:main-problem-logistic-r}
and
\eqref{eq:main-problem-bernoulli}--\eqref{eq:main-problem-bernoulli-r}.
Toward a solution of these problems, 
we understand that a heuristic algorithm that alternates between the optimization of $\theta$ and $x$, 
aside form solving, in each iteration, a problem that is NP-hard for equivalence relations and linear orders,
can have sub-optimal fixed-points.
We also understand that the continuous relaxations of the problems are not necessarily (and not typically) convex.
We have stated these problems as mixed-integer nonlinear programs in order to foster the exchange of ideas between the machine learning community and the optimization communities.

%A distinction between \emph{evidence of elements} where $z$ is defined by inequalities in single variables alone, and \emph{evidence of structure} where $z$ is defined by at least one non-redundant inequality in more than one variable is unnecessary.
%However, evidence of structure has a distinct probabilistic effect:
%In general, the model depicted in
%Fig.~\ref{figure:bayesian-model-a}
%asserts, for any distinct pairs $ab, a'b' \in A \times B$, 
%that $Y_{ab} \ci Y_{a'b'} | X, \Theta$ 
%but not that $Y_{ab} \ci Y_{a'b'} | X, \Theta, Z$.
%Evidence of structure introduces, for at least one distinct pair $ab, a'b' \in A \times B$, a conditional dependency of the form $Y_{ab} \notci Y_{a'b'} | X, \Theta, Z$.

\newpage\appendix

\section{Proofs}
\label{section:appendix:proofs}
\subsection{Proof of Lemma~\ref{lemma:logistic}}
From \eqref{eq:prob-model} follows
\begin{align}
\argmax{\theta \in \mathbb{R}^K, y \in \{0,1\}^{A \times B}}
    & \ p_{Y, \Theta | X, Z}(y, \theta, x, z) \nonumber \\
= \argmax{\theta \in \mathbb{R}^K, y \in z} \hspace{3ex}
    & \  
    \prod_{ab \in A \times B} \underbrace{p_{Y_{ab} | X_{ab}, \Theta}}_{=:p} (y_{ab}, \theta, x_{ab}) 
    \cdot \prod_{k \in K} \underbrace{p_{\Theta_k}}_{=: q}(\theta_k) \nonumber \\
= \argmax{\theta \in \mathbb{R}^K, y \in z} \hspace{3ex}
    & \ 
    \sum_{ab \in A \times B} \log_2 p(y_{ab}, \theta, x_{ab}) 
    + \sum_{k \in K} \log_2 q(\theta_k) \nonumber \\
= \argmin{\theta \in \mathbb{R}^K, y \in z} \hspace{3ex}
    & \  
    - \sum_{ab \in A \times B} \Big(
        y_{ab} \log_2 p(1, \theta, x_{ab}) 
        + (1 - y_{ab}) \log_2 p(0, \theta, x_{ab}) 
    \Big)
    - \sum_{k \in K} \log_2 q(\theta_k) \nonumber \\
= \argmin{\theta \in \mathbb{R}^K, y \in z} \hspace{3ex}
    & \  
    -\sum_{ab \in A \times B} \left(
        y_{ab} \log_2 \frac{p(1, \theta, x_{ab})}{p(0, \theta, x_{ab})}
        + \log_2 p(0, \theta, x_{ab}) 
    \right)
    - \sum_{k \in K} \log_2 q(\theta_k) \nonumber \\
= \argmin{\theta \in \mathbb{R}^K, y \in z} \hspace{3ex}
    & \  
    \sum_{ab \in A \times B} \left(
        y_{ab} \log_2 \frac{p(0, \theta, x_{ab})}{p(1, \theta, x_{ab})}
        - \log_2 p(0, \theta, x_{ab}) 
    \right)
    - \sum_{k \in K} \log_2 q(\theta_k) 
    \enspace .
    \label{eq:main-problem-general-p}
\end{align}

From \eqref{eq:main-problem-general-p} and \eqref{eq:logistic} follows
\begin{align}
\argmax{\theta \in \mathbb{R}^K, y \in \{0,1\}^{A \times B}}
    & \ p_{Y, \Theta | X, Z}(y, \theta, x, z) \nonumber \\
= \argmin{\theta \in \mathbb{R}^K, y \in z} \hspace{3ex}
    & \quad 
    \sum_{ab \in A \times B} \left(
        - y_{ab} \langle \theta, x_{ab} \rangle
        + \log_2 \left(1 + 2^{\langle \theta, x_{ab} \rangle} \right)
    \right)
    + |K| \log_2 (\sigma \sqrt{2 \pi})
    + \frac{\log_2 e}{2 \sigma^2} \|\theta\|_2^2 \nonumber \\
= \argmin{\theta \in \mathbb{R}^K, y \in z} \hspace{3ex}
    & \quad 
    \sum_{ab \in A \times B} \left(
        - y_{ab} \langle \theta, x_{ab} \rangle
        + \log_2 \left(1 + 2^{\langle \theta, x_{ab} \rangle} \right)
    \right)
    + \frac{\log_2 e}{2 \sigma^2} \|\theta\|_2^2
    \enspace .
\end{align}

Partial derivatives of the function $D_x$ defined in
\eqref{eq:main-problem-logistic-d}
with respect to $\theta \in \mathbb{R}^K$ and $y \in (0,1)^{A \times B}$ are
\begin{align}
(\partial_{\theta_j} D_x)(\theta, y)
    & = \sum_{ab \in A \times B} (x_{ab})_j \left(
        - y_{ab} + \frac{1}{1 + 2^{-\langle \theta, x_{ab} \rangle}}
    \right) \\
(\partial_{y_{ab}} D_x)(\theta, y) 
    & = -\langle \theta, x_{ab} \rangle \\
(\partial_{\theta_j, \theta_k} D_x)(\theta, y)
    & = \sum_{ab \in A \times B} (x_{ab})_j (x_{ab})_k
    \underbrace{    
    \frac
        {2^{\langle \theta, x_{ab} \rangle} \log_e 2}
        {\left(1 + 2^{\langle \theta, x_{ab} \rangle}\right)^2}
    }_{=: \xi_{ab}^2} \\
(\partial_{\theta_j, y_{ab}} D_x)(\theta, y) 
= (\partial_{y_{ab}, \theta_j} D_x)(\theta, y)
    & = - (x_{ab})_j \\
(\partial_{y_{ab}, y_{a'b'}} D_x)(\theta, y) 
    & = 0
\enspace .
\end{align}
Thus, the Hessian of $D_x$ is of the special form
\begin{align}
H = \left[\begin{array}{cc} H_{\theta \theta} & H_{\theta y} \\ H_{\theta y}^T & 0 \end{array}\right]
\qquad
\textnormal{with}
\qquad
H_{\theta \theta} = \sum\limits_{ab \in A \times B} \hspace{-1ex} (\xi_{ab} x_{ab}) (\xi_{ab} x_{ab})^T
\enspace .
\end{align}
It defines the quadratic form
\begin{align}
\left[\theta^T\ y^T\right] H \left[\begin{array}{c}\theta\\y\end{array}\right]
    = \log_e(2) \hspace{-2ex} \sum_{ab \in A \times B} 
        \frac
            {2^{\langle \theta, x_{ab} \rangle}}
            {\left(1 + 2^{\langle \theta, x_{ab} \rangle}\right)^2}        
        \langle \theta, x_{ab}\rangle^2
        - 2 \hspace{-1.5ex} \sum_{ab \in A \times B} \hspace{-1.5ex} y_{ab} \langle \theta, x_{ab} \rangle
\enspace .
\label{eq:quadratic-form-hessian-logistic}
\end{align}
This quadratic form need not be positive semi-definite.
Thus, $D_x$ need not be a convex function.
However, the Hessian $H_{\theta \theta}$ of the function $D_x(\cdot, \hat y)$ is positive semi-definite for any fixed $\hat y \in [0,1]^{A \times B}$.
Thus, the function $D_x(\cdot, \hat y)$ is convex.

\subsection{Proof of Lemma~\ref{lemma:bernoulli}}
From \eqref{eq:main-problem-general-p} and \eqref{eq:bernoulli} follows
\begin{align}
\argmax{\theta \in \mathbb{R}^K, y \in \{0,1\}^{A \times B}}
    & \ p_{Y, \Theta | X, Z}(y, \theta, x, z) 
    \nonumber \\
= \argmin{\theta \in \mathbb{R}^K, y \in z} \hspace{3ex}
    & \ 
    \sum_{ab \in A \times B} \left(
        y_{ab} \log_2 \frac{
            \prod_{j \in J} (1 - \theta_j)^{{(x_{ab})}_j}
        }{
            \prod_{j \in J} \theta_j^{{(x_{ab})}_j}
        }
        - \log_2 \prod_{j \in J} (1 - \theta_j)^{{(x_{ab})}_j}
    \right) 
    \nonumber \\
\phantom{= \argmin{\theta \in \mathbb{R}^K, y \in z} \hspace{3ex}}
    & \ 
    - \sum_{j \in J} \log_2 \frac{\Gamma(2\sigma)}{\Gamma^2(\sigma)} \theta_j^{\sigma - 1} (1 - \theta_j)^{\sigma - 1} 
    \nonumber \\
= \argmin{\theta \in \mathbb{R}^K, y \in z} \hspace{3ex}
    & \ 
    \sum_{ab \in A \times B} \left(
        y_{ab} \log_2 \prod_{j \in J} \left(
            \frac{1 - \theta_j}{\theta_j}
        \right)^{{(x_{ab})}_j}
        - \sum_{j \in J} {(x_{ab})}_j \log_2 (1 - \theta_j)
    \right) 
    \nonumber \\
\phantom{= \argmin{\theta \in \mathbb{R}^K, y \in z} \hspace{3ex}}
    & \ 
    - |J| \log_2 \frac{\Gamma(2\sigma)}{\Gamma^2(\sigma)} 
    - (\sigma - 1) \sum_{j \in J} \log_2 \theta_j (1 - \theta_j) 
    \nonumber \\
= \argmin{\theta \in \mathbb{R}^K, y \in z} \hspace{3ex}
    & \ 
    \sum_{ab \in A \times B} \left(
        \left( 
            \sum_{j \in J} {(x_{ab})}_j \log_2 \frac{1 - \theta_j}{\theta_j}
        \right) y_{ab}
        - \sum_{j \in J} {(x_{ab})}_j \log_2 (1 - \theta_j)
    \right) 
    \nonumber \\
\phantom{= \argmin{\theta \in \mathbb{R}^K, y \in z} \hspace{3ex}}
    & \ 
    + (1 - \sigma) \sum_{j \in J} \log_2 \theta_j (1 - \theta_j)
\enspace .
\end{align}

The form $D_x(\theta, y)$ defined in 
\eqref{eq:main-problem-bernoulli-d}
is equivalent to the form below.
\begin{align}
D_x(\theta, y)
& = - \sum_{j \in J} \left(
    \log (1 - \theta_j) \hspace{-1ex} \sum_{ab \in A \times B} \hspace{-1ex} (x_{ab})_j (1 - y_{ab})
    + \log (\theta_j) \hspace{-1ex} \sum_{ab \in A \times B} \hspace{-1ex} (x_{ab})_j y_{ab} 
\right)\\
& = - \sum_{j \in J} \left(
    m_j^- \log (1 - \theta_j)
    + m_j^+ \log \theta_j
\right) 
\enspace .
\end{align}
Partial derivarives of $D_x$ with respect to $\theta \in (0,1)^K$ and $y \in (0,1)^{A \times B}$ are
\begin{align}
(\partial_{\theta_j} D_x)(\theta, y)
    & = \frac{1}{\log_e 2} \left( \frac{m_j^-}{1 - \theta_j} - \frac{m_j^+}{\theta_j} \right) \\
(\partial_{y_{ab}} D_x)(\theta, y) 
    & = \sum_{j \in J} (x_{ab})_j \log_2 \frac{1 - \theta_j}{\theta_j} \\
(\partial_{\theta_j, \theta_k} D_x)(\theta, y)
    & = \frac{\delta_{jk}}{\log_e 2} \left( \frac{m_j^-}{(1 - \theta_j)^2} + \frac{m_j^+}{\theta_j^2} \right) \\ 
(\partial_{\theta_j, y_{ab}} D_x)(\theta, y) 
= (\partial_{y_{ab}, \theta_j} D_x)(\theta, y)
    & = \frac{-1}{\log_e 2} \frac{(x_{ab})_j}{(1 - \theta_j) \theta_j}\\
(\partial_{y_{ab}, y_{a'b'}} D_x)(\theta, y) 
    & = 0
\end{align}
Thus, the Hessian of $D_x$ is of the special form
\begin{align}
H = \left[\begin{array}{cc} H_{\theta \theta} & H_{\theta y} \\ H_{\theta y}^T & 0 \end{array}\right]
\qquad
\textnormal{with}
\qquad
(H_{\theta \theta})_{jk} = \frac{\delta_{jk}}{\log_e 2} \left( \frac{m_j^-}{(1 - \theta_j)^2} + \frac{m_j^+}{\theta_j^2} \right)
\enspace .
\end{align}
It defines the quadratic form
\begin{align}
\left[\theta^T\ y^T\right] H \left[\begin{array}{c}\theta\\y\end{array}\right]
    = \frac{1}{\log_e 2} \sum_{j \in J} \left(
        \frac{\theta_j^2}{(1 - \theta_j)^2} m_j^-
        - \frac{1 + \theta_j}{1 - \theta_j} m_j^+
    \right)
\enspace .
\label{eq:quadratic-form-hessian-bernoulli}
\end{align}
This quadratic form need not be positive semi-definite.
Thus, $D_x$ need not be a convex function.
However, the Hessian $H_{\theta \theta}$ of the function $D_x(\cdot, \hat y)$ is positive semi-definite for any fixed $\hat y \in [0,1]^{A \times B}$.
Thus, the function $D_x(\cdot, \hat y)$ is convex.

\subsection{Proof of Lemma~\ref{lemma:logistic-bound}}
\begin{proof}
Let $ab \in A \times B$ arbitrary and fixed.
If $\hat y_{ab} = 0$,
\begin{align*}
0 \ 
= \ \log_2 1 \ 
< \ \log_2 \left(1 + 2^{\langle \hat\theta, x_{ab} \rangle} \right) \ 
= \ \hat y_{ab} \langle \hat\theta, x_{ab} \rangle + \log_2 \left(1 + 2^{\langle \hat\theta, x_{ab} \rangle}\right)
\enspace .
\end{align*}
If $\hat y_{ab} = 1$,
\begin{align*}
0  
= - \langle x_{ab}, \theta \rangle + \langle x_{ab}, \theta \rangle
= - \hat y_{ab} \langle x_{ab}, \theta \rangle + \log_2 2^{\langle x_{ab}, \theta \rangle}
< - \hat y_{ab} \langle x_{ab}, \theta \rangle + \log_2 \left(1 + 2^{\langle x_{ab}, \theta \rangle} \right)
\enspace .
\end{align*}
That is, every summand in the form 
\eqref{eq:main-problem-logistic-d}
of $D_x$ is bounded from below by 0.
Therefore, $0 < D_x$ and thus, the infimum exists.
Moreover, for any $y \in \{0,1\}^{A \times B}$, $\inf_{\theta, \hat y} D_x(\theta, \hat y) \leq D(0, y) = |A| |B|$, which establishes the upper bound.
\end{proof}

\section{Multilinear Polynomial Lifting}
\label{section:multilinear-polynomial-lifting}
\subsection{Exact}
\begin{definition}
For any finite index set $J$, the \emph{multilinear polynomial lifting} of $\{0,1\}^J$ is the map
$l: \{0,1\}^J \to \{0,1\}^{2^J}$ such that $\forall v \in \{0,1\}^J\ \forall J' \subseteq J$:
\begin{align}
l(v)_{J'} = \prod_{j \in J'} v_j
\enspace .
\label{eq:lifting}
\end{align}
\end{definition}
For example, consider $J = \{1,2\}$ and $l: (v_1, v_2) \mapsto (1, v_1, v_2, v_1 v_2)$.

\begin{lemma}
\label{lemma:boros-hammer}
A one-to-one correspondence between functions $f: \{0,1\}^J \to \mathbb{R}$ and vectors $\theta \in \mathbb{R}^{2^J}$ is established by defining $\forall v \in \{0,1\}^J$:
\begin{align}
f(v) = \langle \theta, l(v)\rangle
\enspace .
\label{eq:boros-hammer}
\end{align}
\end{lemma}
\begin{proof}
By Proposition~2 in 
\cite{boros-2002}.
\end{proof}
For example, consider $J = \{1,2\}$ and $f(v) = \theta_0 + \theta_1 v_1 + \theta_2 v_2 + \theta_{12} v_1 v_2$.

\begin{lemma}
\label{lemma:bijection-logistic}
A one-to-one correspondence between functions $p: \{0,1\}^J \to (0,1)$ and functions $f: \{0,1\}^J \to \mathbb{R}$ is established by defining $\forall v \in \{0,1\}^J$:
\begin{align}
p(v) = \left(1 + 2^{-f(v)}\right)^{-1}
\enspace .
\label{eq:bijection-logistic}
\end{align}
\end{lemma}

\begin{proof}
Trivial.
\end{proof}

With respect to
\eqref{eq:boros-hammer} and \eqref{eq:bijection-logistic}
and the prior in
\eqref{eq:logistic},
the problem 
\eqref{eq:main-problem-general-p}
of estimating a $p_{Y_{ab} | X_{ab}, \Theta}: \{0,1\}^J \to (0,1)$ and a $y \in \{0,1\}^{A \times B}$ so as to maximize
\eqref{eq:prob-model}
can be written in the functional form below
and, thus, in the parametric form
\eqref{eq:main-problem-logistic}--\eqref{eq:main-problem-logistic-r}.
\begin{align}
\min_{f: \{0,1\}^J \to \mathbb{R}} \mathcal{D}(f) + \mathcal{R}_\sigma(f)
\label{eq:main-problem-functional}
\end{align}
\vspace{-2ex}
\begin{align}
\mathcal{D}(f) & :=
\sum_{ab \in A \times B} \left(
    - \langle \theta(f), l(x_{ab})\rangle y_{ab}
    + \log_2 \left(
        1 + 2^{\langle \theta(f), l(x_{ab})\rangle}
    \right)
\right) \\
\mathcal{R}_\sigma(f) & :=
R_\sigma(\theta(f))
\label{eq:main-problem-functional-r}
\end{align}

\subsection{Approximate}
Solving for the $2^{|J|}$ parameters $\theta$ in
\eqref{eq:main-problem-functional}--\eqref{eq:main-problem-functional-r}
is impractical for sufficiently large $|J|$.
To address this problem, 
we approximate the multi-linear polynomial form
\eqref{eq:boros-hammer}
for fixed $d,m \in \mathbb{N}$,
by a linear form $\langle \theta', l'(x) \rangle$ where $l': \{0,1\}^J \to \mathbb{Z}^m$ is drawn randomly from the distribution defined in
\cite{pham2013fast},
such that the inner product $\langle l'(x), l'(x') \rangle$ 
approximates the polynomial kernel $k(x,x') = (1 + \langle x, x' \rangle)^d$.

This approximation of the multi-\emph{variate} polynomial lifting
approximates the multi-\emph{linear} polynomial lifting 
\eqref{eq:lifting}
and thus, the multi-linear polynomial form
\eqref{eq:boros-hammer},
because every multi-linear polynomial form
is a multi-variate polynomial form,
and every multi-variate polynomial form in $\{0,1\}^J$ 
is equivalent to a multi-linear polynomial form in $\{0,1\}^J$ 
(because exponents are irrelevant).

\section{One-Versus-Rest Classification}
\label{section:one-versus-rest}

\begin{lemma}
\label{lemma:map-learning}
Let $v: A \to \{0,1\}^J$ arbitrary and fixed.
Let $x: A \times B \to \{0,1\}^{J \cup B}$ such that, for any $ab \in A \times B$, 
firstly, $(x_{ab})_J = v_J$
and, secondly, for all $b' \in B$, $(x_{ab})_{b'} = 1$ iff $b' = b$.
Let
\begin{align}
\hspace{-1ex}
\mathcal{F} & = \left\{ 
    f: \{0,1\}^{J \cup B} \to \mathbb{R} 
    \ \middle| \ 
    \exists g: B \to \mathbb{R}^{\{0,1\}^J}
    \forall w \in \{0,1\}^{J \cup B}:
        f(w) = \sum_{b \in B} w_b g_b(w_J)
\right\}
\ .
\label{eq:constrained-functions}
\end{align}
Then, $\hat f \in \argmin{f \in \mathcal{F}}\ \mathcal{D}_x(f) + \mathcal{R}_\sigma(f)$ iff,
for every $b \in B$, 
\begin{align}
\hat g_b \in \argmin{g_b: \{0,1\}^J \to \mathbb{R}}\ 
    \sum_{a \in A} \left( 
        - g_b(v_a) y_{ab} + \log_2 \left(
            1 + 2^{g_b(v_a)}
        \right)
    \right)
    + \mathcal{R}_\sigma(g_b)
\enspace .
\end{align}
\end{lemma}

\begin{proof}
Let $J' \subseteq J \cup B$.
If $|J' \cap B| \not= 1$ then $\theta(f)_{J'} = 0$
by \eqref{eq:constrained-functions}.
Otherwise, there exists a unique $b \in J' \cap B$ and $\theta(f)_{J'} = \theta(g_b)_{J'}$
by \eqref{eq:constrained-functions}.
Thus, $\mathcal{R}_\sigma(f) = \sum_{b \in B} \mathcal{R}_\sigma(g_b)$. 
Moreover,
\begin{align*}
\mathcal{D}_x(f)
    & = \sum_{ab \in A \times B} \left(
        - f(x_{ab}) y_{ab}
        + \log_2 \left( 1 + 2^{f(x_{ab})} \right)
    \right) \\
    & = \sum_{ab \in A \times B} \left(
        - \left( \sum_{b' \in B} (x_{ab})_{b'} g_b((x_{ab})_J) \right) y_{ab}
        + \log_2 \left( 1 + 2^{ \left( \sum_{b' \in B} (x_{ab})_{b'} g_b((x_{ab})_J) \right) } \right)
    \right) \\
    & = \sum_{b \in B} \left(
        \sum_{a \in A} \left(
            - g_b(v_a) y_{ab}
            + \log_2 \left( 1 + 2^{g_b(v_a)} \right)
        \right)
    \right)
\enspace .
\end{align*}
\end{proof}

\begin{lemma}
\label{lemma:map-inference}
Let $\hat\theta \in \mathbb{R}^{A \times B}$ arbitrary and fixed.
Call $y \in \{0,1\}^{A \times B}$ a \emph{local solution} 
iff, for all $a \in A$, there exists a $b_a \in B$ such that, 
firstly, $b_a \in \mathrm{argmin}_{b' \in B} - \langle \hat\theta, x_{ab'} \rangle$ and, 
secondly, $\forall b \in B: y_{ab} = 1 \Leftrightarrow b = b_a$.
Then, $y$ is a solution
iff it is a local solution.
\end{lemma}

\begin{proof}
Any local solution is obviously feasible. Any local solution is optimal because
\begin{align*}
& \quad \min_{\{y \in \{0,1\}^{A \times B} | y \in z\}} 
    -\hspace{-1ex}\sum_{ab \in A \times B} \hspace{-1ex} \langle \hat\theta, x_{ab} \rangle y_{ab}\\
= & \quad \sum_{a \in A} \min_{\{y_{a\cdot} \in \{0,1\}^B | y \in z\}}
    - \langle \hat\theta, x_{ab} \rangle y_{ab} \\
= & \quad \sum_{a \in A} \min_{b \in B} 
    - \langle \hat\theta, x_{ab} \rangle y_{ab}
\enspace .
\end{align*}
\end{proof}

\section{Features of Pairs}
\label{section:features-of-pairs}
Consider the problem of estimating a relation on a set $A$, say, an equivalence relation or a linear order.
Instead of a feature vector for every pair $aa' \in A \times A$, that is,
instead of $x: A \times A \to \{0,1\}^J$,
we may be given a feature vector for every element $a \in A$, that is, $w: A \to \{0,1\}^L$.

Now, we need to define, for each pair $aa'$, a feature vector $x_{aa'} \in \{0,1\}^J$ with respect to $w_a$ and $w_{a'}$.
Ideally, $x_{aa'}$ should be invariant under transposition of $w_a$ and $w_{a'}$ and otherwise general. 
Our restriction to $01$-features affords a simple definition which has this property.
For every $01$-feature of elements, indexed by $l \in L$, 
two $01$-features of pairs, indexed by $j_{l1}, j_{l2} \in J$, 
are defined as
\begin{align}
(x_{aa'})_{j_{l1}}    & :=    (v_a)_l (v_{a'})_l \\
(x_{aa'})_{j_{l2}}    & :=    (v_a)_l + (v_{a'})_l - 2 (v_a)_l (v_{a'})_l
\end{align}
These $01$-features are invariant under transposition of $v_a$ and $v_{a'}$. 
Moreover, the multilinear polynomial forms in $x_{aa'}$ comprise the basic transposition invariant multilinear polynomial forms 
\begin{align}
(v_a)_l (v_{a'})_l 
    & = (x_{aa'})_{j_{l1}}\\
(v_a)_l + (v_{a'})_l 
    & = (x_{aa'})_{j_{l2}} + 2 (x_{aa'})_{j_{l1}}
\enspace .
\end{align}

\section{Complementary Experiments}
\label{section:appendix:experiments}

\subsection{Maps (Classification)}
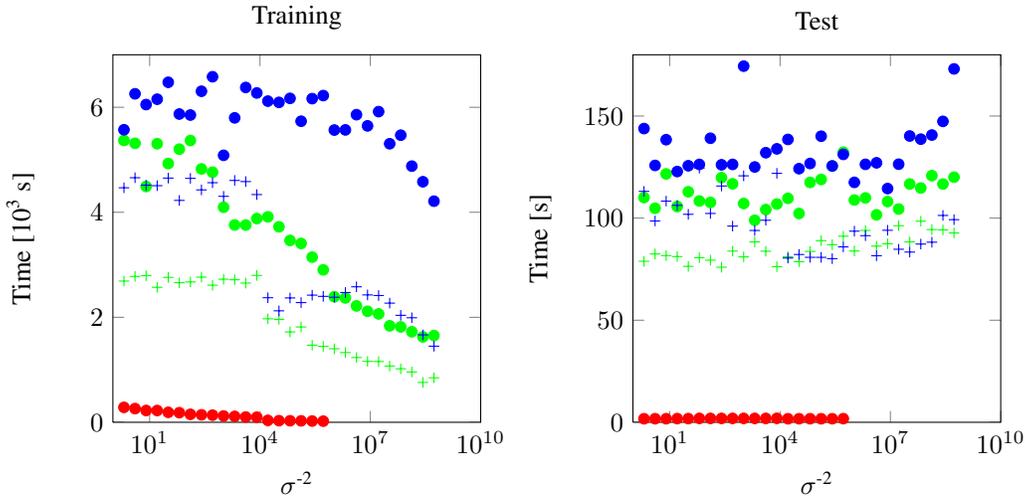
\begin{figure}
\centering
% This file was created by matlab2tikz v0.4.7 (commit faeee56ab2febcb443644b48a5a8422f5522d43f) running on MATLAB 8.3.
% Copyright (c) 2008--2014, Nico Schlömer <nico.schloemer@gmail.com>
% All rights reserved.
% Minimal pgfplots version: 1.3
% 
% The latest updates can be retrieved from
%   http://www.mathworks.com/matlabcentral/fileexchange/22022-matlab2tikz
% where you can also make suggestions and rate matlab2tikz.
% 
\begin{tikzpicture}

\begin{axis}[%
width=0.35\columnwidth,
height=0.35\columnwidth,
unbounded coords=jump,
scale only axis,
xmode=log,
xmin=1,
xmax=10000000000,
xminorticks=true,
xlabel={$\sigma{}^{\text{-2}}$},
ymin=0,
ymax=7,
ylabel={$\text{Time [10}^\text{3}\text{ s]}$},
title={Training},
y tick label style={/pgf/number format/.cd, fixed, fixed zerofill, precision=0, /tikz/.cd}
]
\addplot [color=red,only marks,mark=*,mark options={solid},forget plot]
  table[row sep=crcr]{2	0.284203720414\\
4	0.260753544623\\
8	0.22441207447\\
16	0.22340846351\\
32	0.189975647836\\
64	0.18249870683\\
128	0.151787688638\\
256	0.142160330447\\
512	0.135640800094\\
1024	0.120038378946\\
2048	0.1122643422\\
4096	0.100247338544\\
8192	0.096568557386\\
16384	0.036051876461\\
32768	0.030958555111\\
65536	0.029063398662\\
131072	0.027325693185\\
262144	0.023701850484\\
524288	0.02141608919\\
1048576	nan\\
2097152	nan\\
4194304	nan\\
8388608	nan\\
16777216	nan\\
33554432	nan\\
67108864	nan\\
134217728	nan\\
268435456	nan\\
536870912	nan\\
};
\addplot [color=green,only marks,mark=+,mark options={solid},forget plot]
  table[row sep=crcr]{2	2.689527725048\\
4	2.777867566208\\
8	2.792868188877\\
16	2.570558197226\\
32	2.759974176654\\
64	2.658530569582\\
128	2.673131338365\\
256	2.764555138861\\
512	2.610626079417\\
1024	2.724936436536\\
2048	2.714482300535\\
4096	2.651653761094\\
8192	2.797132433737\\
16384	1.971376879546\\
32768	1.960210397706\\
65536	1.720464120575\\
131072	1.815362476758\\
262144	1.468471340678\\
524288	1.446193257675\\
1048576	1.400741865396\\
2097152	1.327619743186\\
4194304	1.232813699682\\
8388608	1.159819691467\\
16777216	1.15656983643\\
33554432	1.071423488225\\
67108864	1.019031749405\\
134217728	0.958277920333\\
268435456	0.758633306918\\
536870912	0.847055722128\\
};
\addplot [color=green,only marks,mark=*,mark options={solid},forget plot]
  table[row sep=crcr]{2	5.368486929027\\
4	5.313461586052\\
8	4.491530068689\\
16	5.305674411517\\
32	4.928542618084\\
64	5.202020035572\\
128	5.366305704681\\
256	4.825719349934\\
512	4.763606852757\\
1024	4.095337374365\\
2048	3.758079903182\\
4096	3.756923680122\\
8192	3.879341547981\\
16384	3.912872169539\\
32768	3.723830959814\\
65536	3.463204658482\\
131072	3.406014557453\\
262144	3.145656727457\\
524288	2.903640457642\\
1048576	2.390709638668\\
2097152	2.368565252895\\
4194304	2.216495449563\\
8388608	2.114133853687\\
16777216	2.064213713041\\
33554432	1.837612818279\\
67108864	1.816892132465\\
134217728	1.724288320491\\
268435456	1.626182517098\\
536870912	1.652160971903\\
};
\addplot [color=blue,only marks,mark=+,mark options={solid},forget plot]
  table[row sep=crcr]{2	4.466723762653\\
4	4.657699126196\\
8	4.516768716238\\
16	4.5065086025\\
32	4.651277929113\\
64	4.226223519542\\
128	4.645974304175\\
256	4.426434738263\\
512	4.563387724793\\
1024	4.303664703499\\
2048	4.60640788187\\
4096	4.584713897694\\
8192	4.337051646605\\
16384	2.372206216561\\
32768	2.123122686654\\
65536	2.368287126965\\
131072	2.279766759263\\
262144	2.421504106901\\
524288	2.398479210598\\
1048576	2.38187934378\\
2097152	2.472139966916\\
4194304	2.580468888482\\
8388608	2.425976018009\\
16777216	2.414319075099\\
33554432	2.269549273517\\
67108864	2.037622067628\\
134217728	1.990971855117\\
268435456	1.665374527392\\
536870912	1.449702669427\\
};
\addplot [color=blue,only marks,mark=*,mark options={solid},forget plot]
  table[row sep=crcr]{2	5.571411963811\\
4	6.256452356174\\
8	6.051861639294\\
16	6.153429355892\\
32	6.475946128016\\
64	5.871125195833\\
128	5.852151622241\\
256	6.304993905257\\
512	6.581423235123\\
1024	5.08495242614\\
2048	5.798232130964\\
4096	6.377171481141\\
8192	6.272584890664\\
16384	6.117052605845\\
32768	6.093226534311\\
65536	6.169629348264\\
131072	5.734381699187\\
262144	6.164463000191\\
524288	6.222873383388\\
1048576	5.565933483602\\
2097152	5.568872321643\\
4194304	5.85904853499\\
8388608	5.645717908664\\
16777216	5.917202662694\\
33554432	5.305278035079\\
67108864	5.468322986851\\
134217728	4.878987417323\\
268435456	4.58127190217\\
536870912	4.211346565086\\
};
\end{axis}
\end{tikzpicture}%
% This file was created by matlab2tikz v0.4.7 (commit faeee56ab2febcb443644b48a5a8422f5522d43f) running on MATLAB 8.3.
% Copyright (c) 2008--2014, Nico Schlömer <nico.schloemer@gmail.com>
% All rights reserved.
% Minimal pgfplots version: 1.3
% 
% The latest updates can be retrieved from
%   http://www.mathworks.com/matlabcentral/fileexchange/22022-matlab2tikz
% where you can also make suggestions and rate matlab2tikz.
% 
\begin{tikzpicture}

\begin{axis}[%
width=0.35\columnwidth,
height=0.35\columnwidth,
unbounded coords=jump,
scale only axis,
xmode=log,
xmin=1,
xmax=10000000000,
xminorticks=true,
xlabel={$\sigma{}^{\text{-2}}$},
ymin=0,
ymax=180,
ylabel={Time [s]},
title={Test},
y tick label style={/pgf/number format/.cd, fixed, fixed zerofill, precision=0, /tikz/.cd}
]
\addplot [color=red,only marks,mark=*,mark options={solid},forget plot]
  table[row sep=crcr]{2	1.73243278400003\\
4	1.74129217499996\\
8	1.74144965400001\\
16	1.80107117400001\\
32	1.800964169\\
64	1.90000010599999\\
128	1.89275211699999\\
256	1.888167563\\
512	1.90256605499999\\
1024	1.89694334500001\\
2048	1.894503488\\
4096	1.886889567\\
8192	1.934884999\\
16384	1.733342905\\
32768	1.729959758\\
65536	1.73749022\\
131072	1.730392483\\
262144	1.73071987\\
524288	1.77554416500001\\
1048576	nan\\
2097152	nan\\
4194304	nan\\
8388608	nan\\
16777216	nan\\
33554432	nan\\
67108864	nan\\
134217728	nan\\
268435456	nan\\
536870912	nan\\
};
\addplot [color=green,only marks,mark=+,mark options={solid},forget plot]
  table[row sep=crcr]{2	78.850399936\\
4	82.4520475730001\\
8	81.6392337699999\\
16	81.1634346579999\\
32	76.3722548989999\\
64	80.6019707119999\\
128	79.399822721\\
256	75.9486015729999\\
512	83.8311892029997\\
1024	81.0166139809999\\
2048	88.337408632\\
4096	83.798136716\\
8192	76.2236039510003\\
16384	81.0144982430002\\
32768	78.6505392539998\\
65536	83.677952175\\
131072	88.941758942\\
262144	86.971087137\\
524288	91.141570234\\
1048576	83.895683837\\
2097152	93.968018644\\
4194304	86.2725023739999\\
8388608	87.4280387600002\\
16777216	96.276477308\\
33554432	88.335011437\\
67108864	98.4895676679998\\
134217728	94.36607955\\
268435456	94.2196789470001\\
536870912	92.723612752\\
};
\addplot [color=green,only marks,mark=*,mark options={solid},forget plot]
  table[row sep=crcr]{2	110.065302532\\
4	104.825920348\\
8	121.646643741\\
16	105.66730837\\
32	112.903598324\\
64	108.392217354\\
128	107.691830399\\
256	119.804741792999\\
512	116.803008411999\\
1024	107.153549031\\
2048	98.9357634810003\\
4096	104.1824404\\
8192	106.898473345\\
16384	109.688454086\\
32768	102.258567586\\
65536	117.478922182001\\
131072	118.934581089\\
262144	125.534873915\\
524288	132.291448998\\
1048576	108.858313947\\
2097152	109.905903005\\
4194304	101.607054916\\
8388608	108.118244051\\
16777216	104.449156568\\
33554432	116.713800341\\
67108864	114.754259754\\
134217728	120.800798991\\
268435456	116.691002842\\
536870912	120.052432997\\
};
\addplot [color=blue,only marks,mark=+,mark options={solid},forget plot]
  table[row sep=crcr]{2	113.102294877\\
4	98.5917635599999\\
8	108.327435773\\
16	106.249945224999\\
32	101.86397857\\
64	124.491293828\\
128	102.279794765\\
256	115.727688014\\
512	96.0785548379999\\
1024	120.707356816\\
2048	93.9476710430008\\
4096	98.9546083729992\\
8192	121.920996622\\
16384	80.4310748479998\\
32768	82.1807498469998\\
65536	80.795393629\\
131072	80.7814576960004\\
262144	80.1385827270001\\
524288	85.9107722209997\\
1048576	93.6063250069997\\
2097152	91.403202345\\
4194304	81.563666518\\
8388608	94.0438540740001\\
16777216	84.747443533\\
33554432	83.3679357199999\\
67108864	87.2846665090001\\
134217728	88.2018638240002\\
268435456	101.310346513\\
536870912	99.2442232880001\\
};
\addplot [color=blue,only marks,mark=*,mark options={solid},forget plot]
  table[row sep=crcr]{2	143.851259846\\
4	125.810026315999\\
8	138.405054319\\
16	122.80999879\\
32	125.642959405001\\
64	126.298340179999\\
128	139.108634675001\\
256	126.116734143\\
512	126.301741986\\
1024	174.427678567\\
2048	125.045439650999\\
4096	132.028546967\\
8192	133.921535027999\\
16384	138.534721509\\
32768	124.211004349\\
65536	126.754041864\\
131072	140.113119781\\
262144	125.55137662\\
524288	131.201932616999\\
1048576	117.504057131\\
2097152	126.335695492\\
4194304	127.079306539\\
8388608	114.477080406\\
16777216	126.399491404\\
33554432	140.220464896001\\
67108864	138.671872393\\
134217728	140.689217794\\
268435456	147.381912466\\
536870912	173.081682955\\
};
\end{axis}
\end{tikzpicture}%
\vspace{-2ex} % ??? tweaking page break
\caption{Absolute computation times for the classification of images of handwritten digits (MNIST). 
Colors have the same meaning as in 
Fig.~\ref{figure:map-experiment-main}.
Symbols indicate
8192 (${\color{blue}+},{\color{green}+}$)
and 16384 (${\color{blue}\bullet},{\color{green}\bullet}$)
random features, respectively.}
\label{figure:map-runtimes}
\end{figure}
For the classification of images of handwritten digits,
absolute computation times are depicted in 
Fig.~\ref{figure:map-runtimes}.
It can be seen from this figure that it takes 
less than $10^4$ seconds to estimate (learn) all parameters of the probability measure from the entire MNIST training set, using the open-source software
\cite{fan-2008}
to solve the convex learning problem.
It can also be seen from this figure that it takes less than $10^3$ seconds to estimate (infer) the labels of all images of the MNIST test set, using our (trivial) C++ code to solve the (trivial) inference problem.

\subsection{Equivalence Relations (Clustering)}
Toward the clustering of sets of images of handwritten digits,
we reconsider the problem of classifying pairs of images as either showing or not showing the same digit.
A pair $aa'$ of images $a$ and $a'$ is labeled with $y_{aa'} = 1$ if the images show (are labeled with) the same digit.
It is labeled with $y_{aa'} = 0$, otherwise.
Analogous to the experiment described in 
Section~\ref{section:exp:eqr},
we now collect an \emph{unstratified} training set $\{(x_{aa'}, y_{aa'})\}_{aa' \in T}$ by drawing $|T| = 5 \cdot 10^5$ pairs of images randomly, without replacement, from the MNIST training set. 
As the MNIST training set contains (about) equally many images for each of 10 digits, the label $y_{aa'} = 0$ is (about) 9 times as abundant in $T$ as the label $y_{aa'} = 1$. 
A test set of the same cardinality is drawn analogously from the MNIST test set.
Results for the independent classification of pairs (not a solution of the Set Partition Problem) are shown in 
Fig.~\ref{figure:eqr-experiment-main-unstrat}.
It can be seen from this figure that the fraction of misclassified pairs is 
7.45\% for the unstratified test data, at 
$\sigma^{-2} = 2^{21}$ and for an approximation of a multilinear polynomial form of degree $d=2$ by 16384 random features.

For $\hat\theta$ learned with these parameters,
we infer equivalence relations on random subsets $A$ of the MNIST test set 
by solving the Set Partition Problem as described in
Section~\ref{section:exp:eqr}.
An evaluation analogous to 
Section~\ref{section:exp:eqr}
is shown in
Tab.~\ref{table:results-eqr-unstratified}.
In comparison with 
Tab.~\ref{table:results-eqr},
it can be seen that the inferred equivalence relations on previously unseen test sets have a smaller fraction $e_{\ri}$ of misclassified pairs when learning from unstratified (biased) training data.
However, they are worse in terms of the Variation of Information, number of sets and objective value.
This shows empirically that training data in this setting should be stratified.
\begin{figure}
\centering
\input{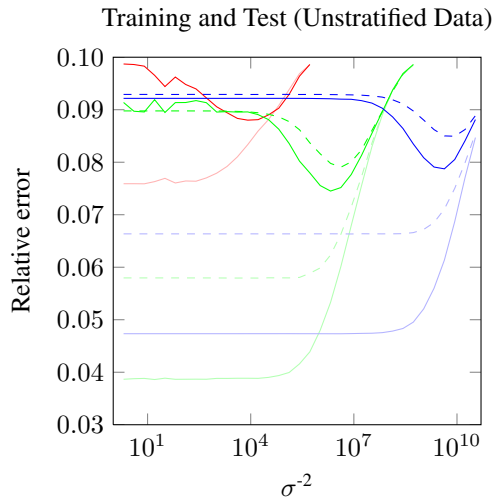}
\caption{Classification of pairs of images of handwritten digits (MNIST). 
Colors and line styles have the same meaning as in
Fig.~\ref{figure:map-experiment-main}.}
\label{figure:eqr-experiment-main-unstrat}
\end{figure}
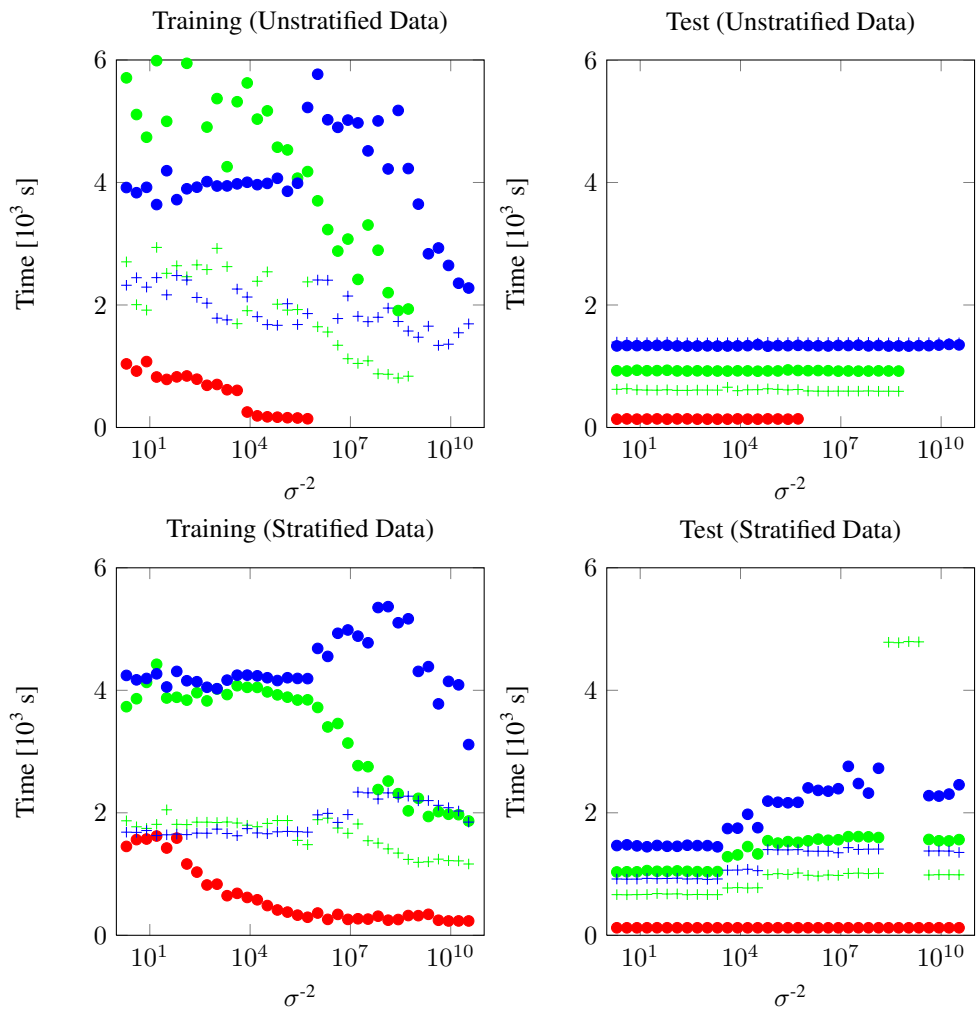
\begin{figure}
\centering
% This file was created by matlab2tikz v0.4.7 (commit 9fae7a376f43b6251e0d59326ee56a5b8a79c469) running on MATLAB 8.3.
% Copyright (c) 2008--2014, Nico Schlömer <nico.schloemer@gmail.com>
% All rights reserved.
% Minimal pgfplots version: 1.3
% 
% The latest updates can be retrieved from
%   http://www.mathworks.com/matlabcentral/fileexchange/22022-matlab2tikz
% where you can also make suggestions and rate matlab2tikz.
% 
\begin{tikzpicture}

\begin{axis}[%
width=0.35\columnwidth,
height=0.35\columnwidth,
unbounded coords=jump,
scale only axis,
xmode=log,
xmin=1,
xmax=100000000000,
xminorticks=true,
xlabel={$\sigma{}^{\text{-2}}$},
ymin=0,
ymax=6,
ylabel={$\text{Time [10}^\text{3}\text{ s]}$},
title={Training (Unstratified Data)}
]
\addplot [color=red,only marks,mark=*,mark options={solid},forget plot]
  table[row sep=crcr]{2	1.03826940864\\
4	0.920954681159\\
8	1.075911738874\\
16	0.824184177253\\
32	0.784775768775\\
64	0.826361791451\\
128	0.841257890553\\
256	0.788558789316\\
512	0.688811843479\\
1024	0.702186815948\\
2048	0.616178153876\\
4096	0.605971739422\\
8192	0.252542732878\\
16384	0.189910633923\\
32768	0.173953939167\\
65536	0.168218749656\\
131072	0.159400134041\\
262144	0.155057425608\\
524288	0.142551941179\\
1048576	nan\\
2097152	nan\\
4194304	nan\\
8388608	nan\\
16777216	nan\\
33554432	nan\\
67108864	nan\\
134217728	nan\\
268435456	nan\\
536870912	nan\\
1073741824	nan\\
2147483648	nan\\
4294967296	nan\\
8589934592	nan\\
17179869184	nan\\
34359738368	nan\\
};
\addplot [color=green,only marks,mark=+,mark options={solid},forget plot]
  table[row sep=crcr]{2	2.70415574782\\
4	2.005290264251\\
8	1.915990265481\\
16	2.939880871987\\
32	2.517106466222\\
64	2.640863959934\\
128	2.461732872601\\
256	2.654814174886\\
512	2.578347517117\\
1024	2.924541773754\\
2048	2.626421110026\\
4096	1.694908252137\\
8192	1.906444185042\\
16384	2.387239800986\\
32768	2.541570820793\\
65536	2.012807773433\\
131072	1.915948901376\\
262144	1.925177274559\\
524288	2.377079420761\\
1048576	1.644548608271\\
2097152	1.558800354109\\
4194304	1.34233854411\\
8388608	1.122925632385\\
16777216	1.045681146781\\
33554432	1.08796438124\\
67108864	0.876364279204\\
134217728	0.868804361563\\
268435456	0.8068934082\\
536870912	0.837238031114\\
1073741824	nan\\
2147483648	nan\\
4294967296	nan\\
8589934592	nan\\
17179869184	nan\\
34359738368	nan\\
};
\addplot [color=green,only marks,mark=*,mark options={solid},forget plot]
  table[row sep=crcr]{2	5.708308759874\\
4	5.109394043867\\
8	4.738773299079\\
16	5.988776425759\\
32	4.998489884888\\
64	31.666645617118\\
128	5.945523725592\\
256	7.039262144645\\
512	4.904262228092\\
1024	5.369314058659\\
2048	4.257349096171\\
4096	5.318141455982\\
8192	5.625135312826\\
16384	5.034604755679\\
32768	5.169538050221\\
65536	4.575905561169\\
131072	4.53169175935\\
262144	4.070699382307\\
524288	4.177733674007\\
1048576	3.700618631434\\
2097152	3.2310092602\\
4194304	2.880259507396\\
8388608	3.075922337728\\
16777216	2.417422215369\\
33554432	3.30552518035\\
67108864	2.894046165604\\
134217728	2.200786870406\\
268435456	1.908790246558\\
536870912	1.934205352968\\
1073741824	nan\\
2147483648	nan\\
4294967296	nan\\
8589934592	nan\\
17179869184	nan\\
34359738368	nan\\
};
\addplot [color=blue,only marks,mark=+,mark options={solid},forget plot]
  table[row sep=crcr]{2	2.32184363913\\
4	2.445356954159\\
8	2.291611537856\\
16	2.446679258778\\
32	2.165264506304\\
64	2.480720838825\\
128	2.40701031438\\
256	2.122537240025\\
512	2.029392842081\\
1024	1.784110785547\\
2048	1.756756366517\\
4096	2.260281079529\\
8192	2.129708649385\\
16384	1.809199877585\\
32768	1.679700506208\\
65536	1.668177821659\\
131072	2.020913005427\\
262144	1.681040114184\\
524288	1.859793885599\\
1048576	2.407942942279\\
2097152	2.404600937925\\
4194304	1.778867017056\\
8388608	2.145911827601\\
16777216	1.815243691912\\
33554432	1.727635170561\\
67108864	1.7993393129\\
134217728	1.9493754399\\
268435456	1.730202032352\\
536870912	1.574776459043\\
1073741824	1.473601459746\\
2147483648	1.653999290592\\
4294967296	1.340472881178\\
8589934592	1.35954753504\\
17179869184	1.546108052514\\
34359738368	1.692100510256\\
};
\addplot [color=blue,only marks,mark=*,mark options={solid},forget plot]
  table[row sep=crcr]{2	3.91742460764\\
4	3.832911416636\\
8	3.92092072715\\
16	3.639357273381\\
32	4.192011078836\\
64	3.719621356276\\
128	3.896942842294\\
256	3.923171881984\\
512	4.014204503432\\
1024	3.942486627927\\
2048	3.943962852453\\
4096	3.977219721722\\
8192	4.003610969566\\
16384	3.963986924288\\
32768	3.984884765\\
65536	4.068460527669\\
131072	3.856483262954\\
262144	3.98808896459\\
524288	5.224596197017\\
1048576	5.767234814518\\
2097152	5.023944618975\\
4194304	4.900195980594\\
8388608	5.018729990396\\
16777216	4.972504593437\\
33554432	4.515994512201\\
67108864	5.005154729982\\
134217728	4.220991842836\\
268435456	5.175938812829\\
536870912	4.227196027287\\
1073741824	3.646370723253\\
2147483648	2.835827981289\\
4294967296	2.930648646412\\
8589934592	2.645818611439\\
17179869184	2.355190216762\\
34359738368	2.276912823839\\
};
\end{axis}
\end{tikzpicture}%
% This file was created by matlab2tikz v0.4.7 (commit 9fae7a376f43b6251e0d59326ee56a5b8a79c469) running on MATLAB 8.3.
% Copyright (c) 2008--2014, Nico Schlömer <nico.schloemer@gmail.com>
% All rights reserved.
% Minimal pgfplots version: 1.3
% 
% The latest updates can be retrieved from
%   http://www.mathworks.com/matlabcentral/fileexchange/22022-matlab2tikz
% where you can also make suggestions and rate matlab2tikz.
% 
\begin{tikzpicture}

\begin{axis}[%
width=0.35\columnwidth,
height=0.35\columnwidth,
unbounded coords=jump,
scale only axis,
xmode=log,
xmin=1,
xmax=100000000000,
xminorticks=true,
xlabel={$\sigma{}^{\text{-2}}$},
ymin=0,
ymax=6,
ylabel={$\text{Time [10}^\text{3}\text{ s]}$},
title={Test (Unstratified Data)}
]
\addplot [color=red,only marks,mark=*,mark options={solid},forget plot]
  table[row sep=crcr]{2	0.136392979851\\
4	0.140002586197\\
8	0.13640732051\\
16	0.137700903114\\
32	0.140078516211\\
64	0.136407684192\\
128	0.139837017861\\
256	0.140028524384\\
512	0.136387408975\\
1024	0.136398197118\\
2048	0.13639876075\\
4096	0.136381328722\\
8192	0.140079616344\\
16384	0.136394817203\\
32768	0.136376984655\\
65536	0.140377216922\\
131072	0.138533063732\\
262144	0.136396395529\\
524288	0.140438583894\\
1048576	nan\\
2097152	nan\\
4194304	nan\\
8388608	nan\\
16777216	nan\\
33554432	nan\\
67108864	nan\\
134217728	nan\\
268435456	nan\\
536870912	nan\\
1073741824	nan\\
2147483648	nan\\
4294967296	nan\\
8589934592	nan\\
17179869184	nan\\
34359738368	nan\\
};
\addplot [color=green,only marks,mark=+,mark options={solid},forget plot]
  table[row sep=crcr]{2	0.622971722302\\
4	0.632776210856\\
8	0.614563573385\\
16	0.610644145963\\
32	0.609086550608\\
64	0.61552490403\\
128	0.606114170108\\
256	0.606067609076\\
512	0.611541168028\\
1024	0.611077901933\\
2048	0.610275065264\\
4096	0.657271107996\\
8192	0.600830615205\\
16384	0.615030187927\\
32768	0.617326823515\\
65536	0.631486353342\\
131072	0.622295591925\\
262144	0.613876668764\\
524288	0.620071405609\\
1048576	0.596529176431\\
2097152	0.593022103568\\
4194304	0.592528407596\\
8388608	0.593625148309\\
16777216	0.593279507803\\
33554432	0.593326000673\\
67108864	0.596250961283\\
134217728	0.591843333303\\
268435456	0.592099207833\\
536870912	0.590608837423\\
1073741824	nan\\
2147483648	nan\\
4294967296	nan\\
8589934592	nan\\
17179869184	nan\\
34359738368	nan\\
};
\addplot [color=green,only marks,mark=*,mark options={solid},forget plot]
  table[row sep=crcr]{2	0.926211516308\\
4	0.921725264359\\
8	0.935050761425\\
16	0.926213975579\\
32	0.927864869853\\
64	0.934388717848\\
128	0.92381514168\\
256	0.922324184519\\
512	0.925208693972\\
1024	0.922482828075\\
2048	0.922213639671\\
4096	0.921405800133\\
8192	0.924790430303\\
16384	0.923828930845\\
32768	0.919383941966\\
65536	0.923285914361\\
131072	0.923068334052\\
262144	0.939671314397\\
524288	0.930325382763\\
1048576	0.927277073667\\
2097152	0.929128298735\\
4194304	0.92680576585\\
8388608	0.920592978543\\
16777216	0.922853190437\\
33554432	0.920263837965\\
67108864	0.921600892916\\
134217728	0.924435813159\\
268435456	0.922779474464\\
536870912	0.921264138151\\
1073741824	nan\\
2147483648	nan\\
4294967296	nan\\
8589934592	nan\\
17179869184	nan\\
34359738368	nan\\
};
\addplot [color=blue,only marks,mark=+,mark options={solid},forget plot]
  table[row sep=crcr]{2	1.394223995429\\
4	1.39104471724\\
8	1.384137348069\\
16	1.385160227103\\
32	1.384611700017\\
64	1.393096432155\\
128	1.384043429099\\
256	1.381836090811\\
512	1.380100029168\\
1024	1.381384695933\\
2048	1.38096988935\\
4096	1.383985331961\\
8192	1.380593837984\\
16384	1.38548775667\\
32768	1.377722187189\\
65536	1.391476220145\\
131072	1.38017229829\\
262144	1.385744465941\\
524288	1.39404669924\\
1048576	1.383869332755\\
2097152	1.386831124259\\
4194304	1.387267942405\\
8388608	1.388192710746\\
16777216	1.386618027216\\
33554432	1.387892434436\\
67108864	1.389333045976\\
134217728	1.380724993081\\
268435456	1.382368007635\\
536870912	1.386447597308\\
1073741824	1.382869968541\\
2147483648	1.386062110387\\
4294967296	1.384766671196\\
8589934592	1.383705309965\\
17179869184	1.373080191377\\
34359738368	1.38192261503\\
};
\addplot [color=blue,only marks,mark=*,mark options={solid},forget plot]
  table[row sep=crcr]{2	1.331878727782\\
4	1.338378716193\\
8	1.335666311016\\
16	1.337659576908\\
32	1.340368697299\\
64	1.341896608822\\
128	1.3317556992\\
256	1.329905542146\\
512	1.331089732549\\
1024	1.332822042741\\
2048	1.329712449213\\
4096	1.330693466635\\
8192	1.332583098818\\
16384	1.335879364405\\
32768	1.352422541116\\
65536	1.327772741909\\
131072	1.336628550287\\
262144	1.335352010427\\
524288	1.340950259044\\
1048576	1.337293489056\\
2097152	1.334096462313\\
4194304	1.331245731481\\
8388608	1.339983781282\\
16777216	1.337986978947\\
33554432	1.344252945442\\
67108864	1.3326681936\\
134217728	1.339561771887\\
268435456	1.330440840931\\
536870912	1.33205072872\\
1073741824	1.329501356017\\
2147483648	1.336179412736\\
4294967296	1.335810661327\\
8589934592	1.345392199315\\
17179869184	1.356236745764\\
34359738368	1.348899809479\\
};
\end{axis}
\end{tikzpicture}%
\\
% This file was created by matlab2tikz v0.4.7 (commit faeee56ab2febcb443644b48a5a8422f5522d43f) running on MATLAB 8.3.
% Copyright (c) 2008--2014, Nico Schlömer <nico.schloemer@gmail.com>
% All rights reserved.
% Minimal pgfplots version: 1.3
% 
% The latest updates can be retrieved from
%   http://www.mathworks.com/matlabcentral/fileexchange/22022-matlab2tikz
% where you can also make suggestions and rate matlab2tikz.
% 
\begin{tikzpicture}

\begin{axis}[%
width=0.35\columnwidth,
height=0.35\columnwidth,
scale only axis,
xmode=log,
xmin=1,
xmax=100000000000,
xminorticks=true,
xlabel={$\sigma{}^{\text{-2}}$},
ymin=0,
ymax=6,
ylabel={$\text{Time [10}^\text{3}\text{ s]}$},
title={Training (Stratified Data)}
]
\addplot [color=red,only marks,mark=*,mark options={solid},forget plot]
  table[row sep=crcr]{2	1.45109191894531\\
4	1.56213415527344\\
8	1.57155444335938\\
16	1.62047521972656\\
32	1.42493444824219\\
64	1.58680297851563\\
128	1.16417297363281\\
256	1.03164013671875\\
512	0.821622131347656\\
1024	0.834716918945313\\
2048	0.646492004394531\\
4096	0.684795532226563\\
8192	0.617847839355469\\
16384	0.57961865234375\\
32768	0.485257965087891\\
65536	0.413484832763672\\
131072	0.37882080078125\\
262144	0.329003540039063\\
524288	0.294615234375\\
1048576	0.363449340820312\\
2097152	0.259949188232422\\
4194304	0.341241302490234\\
8388608	0.259665100097656\\
16777216	0.269780151367188\\
33554432	0.2643115234375\\
67108864	0.31174853515625\\
134217728	0.246824630737305\\
268435456	0.257383331298828\\
536870912	0.322404663085938\\
1073741824	0.320632629394531\\
2147483648	0.342630981445313\\
4294967296	0.245703384399414\\
8589934592	0.233975372314453\\
17179869184	0.234448654174805\\
34359738368	0.233957107543945\\
};
\addplot [color=green,only marks,mark=+,mark options={solid},forget plot]
  table[row sep=crcr]{2	1.86944055175781\\
4	1.77178918457031\\
8	1.74520031738281\\
16	1.81121411132813\\
32	2.04869458007812\\
64	1.81141870117187\\
128	1.80800500488281\\
256	1.84745007324219\\
512	1.84705224609375\\
1024	1.8414638671875\\
2048	1.8495712890625\\
4096	1.83299731445313\\
8192	1.79687451171875\\
16384	1.77309606933594\\
32768	1.82674682617187\\
65536	1.87373229980469\\
131072	1.87367919921875\\
262144	1.54916125488281\\
524288	1.47980981445313\\
1048576	1.88742297363281\\
2097152	1.91175256347656\\
4194304	1.7510810546875\\
8388608	1.66523095703125\\
16777216	1.81734020996094\\
33554432	1.54363427734375\\
67108864	1.50506298828125\\
134217728	1.41084716796875\\
268435456	1.34230639648438\\
536870912	1.23888037109375\\
1073741824	1.19161022949219\\
2147483648	1.20001416015625\\
4294967296	1.24339233398438\\
8589934592	1.218828125\\
17179869184	1.21486962890625\\
34359738368	1.16406372070313\\
};
\addplot [color=green,only marks,mark=*,mark options={solid},forget plot]
  table[row sep=crcr]{2	3.73224975585938\\
4	3.86243432617188\\
8	4.1292001953125\\
16	4.42534716796875\\
32	3.87446142578125\\
64	3.88591259765625\\
128	3.83960473632813\\
256	3.96092651367187\\
512	3.82717236328125\\
1024	4.02079809570312\\
2048	3.93028393554687\\
4096	4.07504223632812\\
8192	4.04837036132812\\
16384	4.04728955078125\\
32768	3.97395922851562\\
65536	3.92548999023438\\
131072	3.88805737304688\\
262144	3.84108618164062\\
524288	3.84446752929688\\
1048576	3.7196865234375\\
2097152	3.4019208984375\\
4194304	3.45766479492187\\
8388608	3.13850830078125\\
16777216	2.76989331054688\\
33554432	2.75253540039063\\
67108864	2.37957421875\\
134217728	2.51812915039062\\
268435456	2.31245263671875\\
536870912	2.03178283691406\\
1073741824	2.23645483398438\\
2147483648	1.93941540527344\\
4294967296	2.01856713867187\\
8589934592	1.97572277832031\\
17179869184	1.96944018554688\\
34359738368	1.86616674804688\\
};
\addplot [color=blue,only marks,mark=+,mark options={solid},forget plot]
  table[row sep=crcr]{2	1.68373046875\\
4	1.68040490722656\\
8	1.710607421875\\
16	1.62823181152344\\
32	1.64097961425781\\
64	1.6521083984375\\
128	1.64402209472656\\
256	1.66781091308594\\
512	1.66644543457031\\
1024	1.73326123046875\\
2048	1.6620595703125\\
4096	1.62559448242188\\
8192	1.74174755859375\\
16384	1.67065161132812\\
32768	1.65829479980469\\
65536	1.68914465332031\\
131072	1.69543334960937\\
262144	1.69182397460938\\
524288	1.68423254394531\\
1048576	1.96690478515625\\
2097152	1.99033020019531\\
4194304	1.84242395019531\\
8388608	1.97117309570313\\
16777216	2.33769750976563\\
33554432	2.32585522460937\\
67108864	2.22445825195312\\
134217728	2.3248701171875\\
268435456	2.25204272460938\\
536870912	2.2712294921875\\
1073741824	2.19547827148438\\
2147483648	2.1995458984375\\
4294967296	2.12180810546875\\
8589934592	2.083404296875\\
17179869184	2.02926928710938\\
34359738368	1.84774963378906\\
};
\addplot [color=blue,only marks,mark=*,mark options={solid},forget plot]
  table[row sep=crcr]{2	4.2431748046875\\
4	4.1707666015625\\
8	4.1934345703125\\
16	4.269248046875\\
32	4.05427490234375\\
64	4.308716796875\\
128	4.15582080078125\\
256	4.14007568359375\\
512	4.04962548828125\\
1024	4.02612084960938\\
2048	4.1659580078125\\
4096	4.246984375\\
8192	4.24767529296875\\
16384	4.236572265625\\
32768	4.204833984375\\
65536	4.15919921875\\
131072	4.20514892578125\\
262144	4.19366259765625\\
524288	4.19080908203125\\
1048576	4.683578125\\
2097152	4.55334326171875\\
4194304	4.9309638671875\\
8388608	4.98495068359375\\
16777216	4.88441357421875\\
33554432	4.775873046875\\
67108864	5.35102197265625\\
134217728	5.36577294921875\\
268435456	5.10280615234375\\
536870912	5.16743212890625\\
1073741824	4.3087119140625\\
2147483648	4.38646728515625\\
4294967296	3.77841333007812\\
8589934592	4.14565087890625\\
17179869184	4.08920361328125\\
34359738368	3.11402075195313\\
};
\end{axis}
\end{tikzpicture}%
% This file was created by matlab2tikz v0.4.7 (commit faeee56ab2febcb443644b48a5a8422f5522d43f) running on MATLAB 8.3.
% Copyright (c) 2008--2014, Nico Schlömer <nico.schloemer@gmail.com>
% All rights reserved.
% Minimal pgfplots version: 1.3
% 
% The latest updates can be retrieved from
%   http://www.mathworks.com/matlabcentral/fileexchange/22022-matlab2tikz
% where you can also make suggestions and rate matlab2tikz.
% 
\begin{tikzpicture}

\begin{axis}[%
width=0.35\columnwidth,
height=0.35\columnwidth,
scale only axis,
xmode=log,
xmin=1,
xmax=100000000000,
xminorticks=true,
xlabel={$\sigma{}^{\text{-2}}$},
ymin=0,
ymax=6,
ylabel={$\text{Time [10}^\text{3}\text{ s]}$},
title={Test (Stratified Data)}
]
\addplot [color=red,only marks,mark=*,mark options={solid},forget plot]
  table[row sep=crcr]{2	0.122775303204\\
4	0.122346331664\\
8	0.122415628541\\
16	0.122857544584\\
32	0.123177429254\\
64	0.123133139814\\
128	0.123226998281\\
256	0.122737195577\\
512	0.122742619882\\
1024	0.123044246839\\
2048	0.12288625301\\
4096	0.122776556266\\
8192	0.122800163949\\
16384	0.122437259809\\
32768	0.122813673485\\
65536	0.122674767292\\
131072	0.123250217422\\
262144	0.123249856748\\
524288	0.123227014837\\
1048576	0.123222620175\\
2097152	0.12274807844\\
4194304	0.12275495135\\
8388608	0.122748914144\\
16777216	0.123052919551\\
33554432	0.123236540959\\
67108864	0.123159568525\\
134217728	0.123187741759\\
268435456	0.1227546631\\
536870912	0.122752389327\\
1073741824	0.122752708363\\
2147483648	0.123087994052\\
4294967296	0.123235333095\\
8589934592	0.123248568906\\
17179869184	0.122758511408\\
34359738368	0.122756373539\\
};
\addplot [color=green,only marks,mark=+,mark options={solid},forget plot]
  table[row sep=crcr]{2	0.663574547706\\
4	0.661338947037\\
8	0.664036216044\\
16	0.666558763855\\
32	0.681297535409\\
64	0.673732463582\\
128	0.676267592084\\
256	0.663052224376\\
512	0.666022216714\\
1024	0.663782707169\\
2048	0.66311744869\\
4096	0.772209296017\\
8192	0.777732013918\\
16384	0.772070710836\\
32768	0.775221528482\\
65536	0.990573796158\\
131072	1.005855478593\\
262144	0.994342496859\\
524288	1.016004857079\\
1048576	0.97962810662\\
2097152	0.967287589701\\
4194304	0.983850397943\\
8388608	0.977459914333\\
16777216	1.009917260875\\
33554432	1.015011831777\\
67108864	1.007588140932\\
134217728	1.012792617638\\
268435456	4.784174946902\\
536870912	4.777174220988\\
1073741824	4.795353648069\\
2147483648	4.790696877663\\
4294967296	0.984547661089\\
8589934592	0.989503428514\\
17179869184	0.989020919085\\
34359738368	0.987654982397\\
};
\addplot [color=green,only marks,mark=*,mark options={solid},forget plot]
  table[row sep=crcr]{2	1.034120407995\\
4	1.036775710627\\
8	1.036183132979\\
16	1.05621396038\\
32	1.04492913488\\
64	1.044146006725\\
128	1.051542882333\\
256	1.044510087923\\
512	1.039369437791\\
1024	1.037181556852\\
2048	1.04017860002\\
4096	1.282472710513\\
8192	1.312862971567\\
16384	1.448407902881\\
32768	1.328375996275\\
65536	1.544454953522\\
131072	1.505407424797\\
262144	1.529365311157\\
524288	1.518319612194\\
1048576	1.542515612606\\
2097152	1.570428898291\\
4194304	1.551523347482\\
8388608	1.559569831769\\
16777216	1.610561639142\\
33554432	1.611250250506\\
67108864	1.607301072285\\
134217728	1.595880595081\\
268435456	7.702824701884\\
536870912	7.65653478672\\
1073741824	7.672137106676\\
2147483648	7.705583510843\\
4294967296	1.564714758738\\
8589934592	1.543089599185\\
17179869184	1.540800247598\\
34359738368	1.561962148086\\
};
\addplot [color=blue,only marks,mark=+,mark options={solid},forget plot]
  table[row sep=crcr]{2	0.919149049169\\
4	0.916539812129\\
8	0.917557790886\\
16	0.928833905404\\
32	0.924594934351\\
64	0.927370340282\\
128	0.931611145099\\
256	0.919228283931\\
512	0.922529323708\\
1024	0.913565935106\\
2048	0.919569004593\\
4096	1.063212262376\\
8192	1.065087441144\\
16384	1.077927263638\\
32768	1.055244223831\\
65536	1.39902396063\\
131072	1.394467782341\\
262144	1.394479513254\\
524288	1.398029333821\\
1048576	1.373564551817\\
2097152	1.371847711451\\
4194304	1.371209127748\\
8388608	1.347668612545\\
16777216	1.431329683892\\
33554432	1.401671832787\\
67108864	1.40605990279\\
134217728	1.407406082432\\
268435456	6.792685471253\\
536870912	6.677815236295\\
1073741824	6.691106026262\\
2147483648	6.675591758776\\
4294967296	1.376029847025\\
8589934592	1.37664234065\\
17179869184	1.374728712798\\
34359738368	1.352409995696\\
};
\addplot [color=blue,only marks,mark=*,mark options={solid},forget plot]
  table[row sep=crcr]{2	1.464830699132\\
4	1.47663953303\\
8	1.460165707157\\
16	1.445326578751\\
32	1.466609730368\\
64	1.453246502678\\
128	1.452844723409\\
256	1.470300425887\\
512	1.464276234193\\
1024	1.463836254086\\
2048	1.444541036716\\
4096	1.74270794832\\
8192	1.748893331176\\
16384	1.975792274393\\
32768	1.756843424915\\
65536	2.190623455743\\
131072	2.172334982807\\
262144	2.163210128231\\
524288	2.170899970648\\
1048576	2.406361330905\\
2097152	2.365537316876\\
4194304	2.353072971501\\
8388608	2.393135484853\\
16777216	2.758347031622\\
33554432	2.477846137353\\
67108864	2.321504317693\\
134217728	2.726638756347\\
268435456	11.445987036662\\
536870912	11.478419653668\\
1073741824	11.51273677372\\
2147483648	11.397288291722\\
4294967296	2.279362085228\\
8589934592	2.272593046455\\
17179869184	2.305382860665\\
34359738368	2.457449451277\\
};
\end{axis}
\end{tikzpicture}%
\caption{Absolute computation times for the classification of pairs of handwritten digits (learning and inference).}
\label{figure:eqr-runtimes}
\end{figure}
\begin{table}
\caption{Comparison of equivalence relations on unstratified random subsets $A$ of the MNIST test set}
\label{table:results-eqr-unstratified}
\begin{center}
\begin{small}
\begin{tabular}{@{}l@{\hspace{1.5ex}}l@{\hspace{1.5ex}}r@{\hspace{1.5ex}}r@{\hspace{1.5ex}}r@{\hspace{1.5ex}}r@{\hspace{1.5ex}}r@{\hspace{1.5ex}}r@{\hspace{1.5ex}}r@{\ }r@{\ }r@{}}
\hline\\[-1ex]

& 
& \multicolumn{1}{@{}l}{$|A|$}
& \multicolumn{1}{@{}l}{${|A| \choose 2}$}
& \multicolumn{1}{@{}l}{$e_{\mathrm{RI}}\ [\%]$}
& \multicolumn{1}{@{}l}{$\vi$ \cite{meila-2003}}
& \multicolumn{1}{@{}l}{Sets}
& \multicolumn{1}{@{}l}{Obj./${|A| \choose 2} \cdot 10^2$} 
& \multicolumn{3}{@{}l}{$t\ [s]$}\\[1ex]

\hline\\[-1ex]

\multirow{11}{*}{$\hat\theta$} & \multirow{5}{*}{$\hat y$} & 100   & 4950   & $6.36$ $\pm$ $0.74$   & $\bf 1.35$ $\pm$ $0.19$   & $43.0$ $\pm$ $5.8$   & $-3.84$ $\pm$ $1.17$   & $\bf 3.23$   & $\pm$    & $2.29$\\
& & 170   & 14365   & $6.48$ $\pm$ $0.38$   & $\bf 1.56$ $\pm$ $0.10$   & $64.1$ $\pm$ $5.7$   & $-3.76$ $\pm$ $0.66$   & $\bf 5.75$   & $\pm$    & $4.21$\\
& & 220   & 24090   & $6.81$ $\pm$ $0.41$   & $\bf 1.75$ $\pm$ $0.12$   & $77.8$ $\pm$ $3.5$   & $-3.51$ $\pm$ $0.55$   & $\bf 12.05$   & $\pm$    & $9.25$\\
& & 260   & 33670   & $6.85$ $\pm$ $0.35$   & $\bf 1.82$ $\pm$ $0.12$   & $89.4$ $\pm$ $5.9$   & $-3.46$ $\pm$ $0.53$   & $\bf 26.97$   & $\pm$    & $26.25$\\
& & 300   & 44850   & $6.57$ $\pm$ $0.22$   & $\bf 1.79$ $\pm$ $0.09$   & $94.8$ $\pm$ $7.2$   & $-3.68$ $\pm$ $0.30$   & $\bf 107.71$   & $\pm$    & $129.69$\\[1ex]

& \multirow{6}{*}{$\hat y^{\kl}$} & 100   & 4950   & $6.40$ $\pm$ $0.76$   & $\bf 1.36$ $\pm$ $0.19$   & $43.0$ $\pm$ $5.9$   & $-3.84$ $\pm$ $1.17$   & $\bf 0.01$   & $\pm$    & $0.00$\\
& & 170   & 14365   & $6.46$ $\pm$ $0.42$   & $\bf 1.56$ $\pm$ $0.12$   & $63.8$ $\pm$ $5.9$   & $-3.75$ $\pm$ $0.67$   & $\bf 0.03$   & $\pm$    & $0.01$\\
& & 220   & 24090   & $6.80$ $\pm$ $0.44$   & $\bf 1.75$ $\pm$ $0.13$   & $77.3$ $\pm$ $4.1$   & $-3.50$ $\pm$ $0.56$   & $\bf 0.06$   & $\pm$    & $0.02$\\
& & 260   & 33670   & $6.85$ $\pm$ $0.37$   & $\bf 1.83$ $\pm$ $0.14$   & $89.4$ $\pm$ $6.2$   & $-3.46$ $\pm$ $0.53$   & $\bf 0.09$   & $\pm$    & $0.03$\\
& & 300   & 44850   & $6.55$ $\pm$ $0.22$   & $\bf 1.78$ $\pm$ $0.09$   & $94.5$ $\pm$ $7.0$   & $-3.68$ $\pm$ $0.30$   & $\bf 0.15$   & $\pm$    & $0.04$\\

& & $\bf 10^4$ & $\bf 5 \cdot 10^9$ & \multicolumn{1}{@{}l}{$\bf 6.69$} & \multicolumn{1}{@{}l}{$\bf 2.88$} & \multicolumn{1}{@{}l}{$\bf 1168$} & $\bf -3.74$ & $\bf 1340.70$\\[1ex]

\hline
\end{tabular}
\end{small}
\end{center}
\end{table}

\subsection{Orders (Ranking)}
\begin{figure}
\centering
\includegraphics{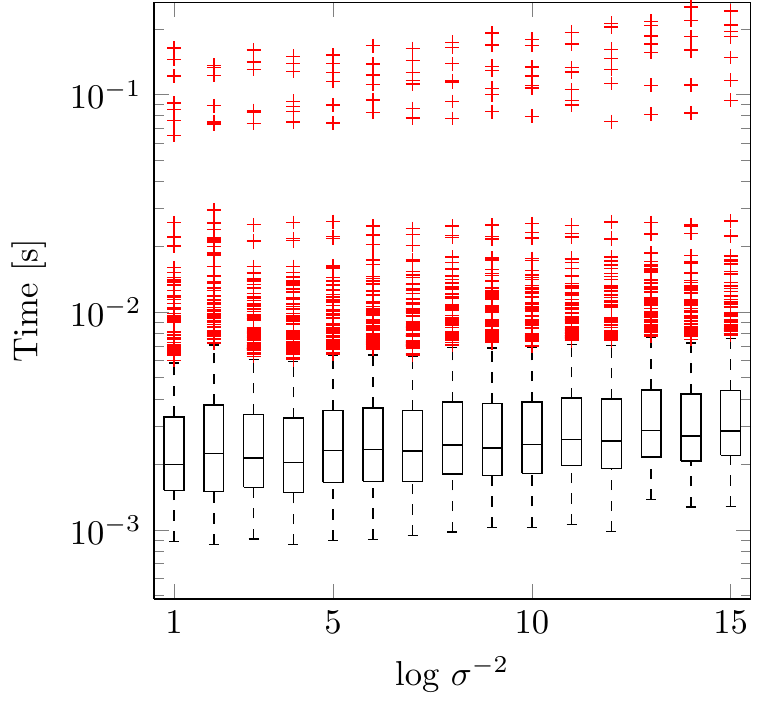}
\caption{Absolute computation time for linear ordering of words in sentences (inference).}
\label{figure:order-runtimes}
\end{figure}
For the linear ordering of words in sentences, absolute computation times are summarized in
Fig.~\ref{figure:order-runtimes}.

\clearpage
\bibliographystyle{unsrt}
\bibliography{nips2014}

\end{document}